\documentclass[reqno,10pt]{amsart}
\usepackage{graphicx} 

\usepackage{graphicx,amsfonts,amssymb,amsmath,latexsym,amsthm,fullpage, framed,url,tikz,comment,cancel,mathtools,amsthm, hyperref} 
\usepackage[normalem]{ulem}
\usepackage{centernot}
\usepackage{multicol}
\usepackage{hyperref}
\usepackage{dsfont}
\usepackage[T1]{fontenc}
\usepackage{mathrsfs}

\newcommand{\e}{\varepsilon}
\newcommand{\eps}{\varepsilon}

\newcommand{\R}{\mathbb{R}}

\newcommand{\N}{\mathbb{N}}

\newcommand{\B}{\mathbb{B}}
\newcommand{\A}{\mathbb{A}}

\newtheorem{theorem}{Theorem}[section]
\newtheorem{lemma}[theorem]{Lemma}
\newtheorem{proposition}[theorem]{Proposition}

\newtheorem{definition}[theorem]{Definition}

\newtheorem{remark}[theorem]{Remark}
\newtheorem{example}[theorem]{Example}

\numberwithin{equation}{section}

\begin{document}
\frenchspacing

\title[Robustness and accuracy in adversarial training]{On the tradeoff between robustness and geometric accuracy \\ in Adversarial Training}
\author[R. Morris] {Rachel Morris} 
\address[Rachel Morris]{Department of Mathematics and Statistics, Concordia University, Montreal, Canada}
\email{rachel.morris@mail.concordia.ca }
\author[R. Murray] {Ryan Murray} 
\address[Ryan Murray]{Department of Mathematics, North Carolina State University, Raleigh, USA}
\email{rwmurray@ncsu.edu}
\author[K. Stinson] {Kerrek Stinson} 
\address[Kerrek Stinson]{Department of Mathematics, University of Utah, Salt Lake City, USA}
\email{kerrek.stinson@utah.edu}


\begin{abstract}   
{For non-parametric binary classification in the presence of an adversary, we prove an optimal quantitative estimate for adversarial classifiers controlling geometric accuracy, as measured by the Hausdorff distance to the Bayes classifier, in terms or their robustness, as encoded by the adversarial budget $\eps$.}
This rigorously shows that robust classifiers sacrifice at most $O(\sqrt{\eps})$ or $O(\eps)$ in accuracy, depending on the regularity of the underlying Bayes classifier. In contrast to prior work on the incompatibility of robustness and accuracy over low-complexity hypothesis classes, our work shows that high-complexity hypothesis classes allow for robust classifiers with precisely quantified accuracy. Our results apply for adversaries determined by any norm on Euclidean space and in both the infinite and finite data settings.

\end{abstract}

\maketitle

\section{Introduction}

Adversarial training is a robust optimization principle that accounts for perturbations of the data in the training process. In the formulation proposed by Madry et al.\ \cite{madry2018}, one seeks to solve the min-max optimization problem
$$ \text{adversarial training} \qquad \min_{u}\,\mathbb{E}_{(x,y)\sim \mu} \Big[\sup_{\tilde x \in {B(x,\eps)}}\ell (u(\tilde x),y)\Big], $$ where $u$ is a classifier, $\mu$ is a probability measure describing the data, $\ell$ is the loss function, ${B(x,\eps)}$ is a ball defined by a norm $\sigma$, and $\eps>0$ represents the adversarial budget . As seen in a variety of work \cite{bai2021recent,shafahi2019adversarial,xhonneux2024efficient,zhang2019theoretically}, adversarial training is a powerful antidote to adversarial attacks. However, there is an underlying tension in the problem: 
$$\text{Does robustness come at the expense of accuracy?}$$
Here, accuracy is often measured by the standard risk $\mathbb{E}_{(x,y)\sim \mu} [\ell (u( x),y)]$. 
We are not the first to question the tradeoff between robustness and accuracy. Tsipras et al.\ \cite{tsipras2019robustness} constructed a binary classification problem where small standard risk is \emph{incompatible} with small robust risk over all linear classifiers. Nakkiran \cite{nakkiran2019} further quantified this by showing that 
$$(\text{$\eps$-robust risk}) + (\text{standard risk})^\gamma \geq 1 \quad \text{ for some }\gamma>0$$ over linear classifiers for a particular classification problem.  Instead of viewing this as an intrinsic shortcoming of adversarial training, Nakkiran proposed that robust training is at odds with simplicity, and the loss of accuracy is a consequence of an insufficiently complex hypothesis class.  In other words, finding a robust and (standard) accurate classifier is prohibited by a rigid hypothesis class.

In this paper, we take the complementary view: We show that given a flexible hypothesis class, robust classifiers sacrifice at most $O(\sqrt{\eps})$ in \emph{geometric accuracy}, and in fact, the accuracy is guaranteed to improve with additional structural information on the decision boundary of the Bayes classifier. Precisely, in the setting of non-parametric (where classifiers are the most expressive) binary classification, we consider the \emph{Bayes classifier}
$$A_0 = \operatorname{argmin}_{A \subset \R^d} \mathbb{E}_{(x,y)\sim \mu} \Big[|\chi_{A}( x) - y|\Big],$$
which is unique under a simple non-degeneracy assumption on the data $\mu$,
and an \emph{adversarial classifier} 
$$A_\eps \in \operatorname{argmin}_{A \subset \R^d} \mathbb{E}_{(x,y)\sim \mu} \Big[\sup_{\tilde x \in {B(x,\eps)}}|\chi_{A}(\tilde x) - y|\Big].$$
We prove that for an $\ell^2$-adversary it holds that
\begin{equation}\label{eqn:introMain}
d_{\mathcal{H},K}({A}_\eps,A_0) \lesssim \begin{cases}
\sqrt{\eps} & \text{ when $\partial A_0$ is locally Lipschitz},\\
\eps & \text{ when $\partial A_0$ is locally $C^{1,1}$-regular},
\end{cases} 
\end{equation}
where $d_{\mathcal{H},K}$ is the local Hausdorff metric on $K\subset \subset \R^d$ (in fact, a related, but weaker estimate \eqref{eqn:noreg} holds with no regularity assumptions on $\partial A_0$). A similar geometric estimate holds for adversarial training with perturbations determined by any norm $\sigma$ (e.g., an $\ell^\infty$-adversary) and for the empirical adversarial training problem with near optimal sample complexity; we refer to Theorems \ref{thm:main} and \ref{thm:mainEmpirical} for the precise details. 
We remark that Raghunathan et al.\ \cite{raghunathan2019} were concerned with the robustness-accuracy tradeoff in finite data settings, and our empirical result (Theorem \ref{thm:mainEmpirical}) provides, under appropriate assumptions, rigorous guarantees in precisely finite data settings.

Our approach builds on the geometric perspective for adversarial training introduced in \cite{bungert2023geom}. Therein, it was shown that the adversarial training problem could equivalently be viewed as standard risk minimization regularized by a nonlocal total variation, allowing the authors to prove the existence of adversarial classifiers. Subsequent work by Bungert and the last author in \cite{bungert2024gammaconv} identified the asymptotic limit of the nonlocal total variation as the adversarial budget vanished. Using this, they showed that even without non-degeneracy assumptions on the data (i.e., allowing for non-unique Bayes classifiers), adversarial training acts as a selection mechanism for standard risk minimization with adversarial classifiers converging to
a minimal-perimeter Bayes classifier as $\eps\to 0$. It was very recently shown that this selection is still present for parametric models of adversarial training and allows for interpretability of the model via the saliency map \cite{lunk2026}.
Quantitative estimates for geometric accuracy were first introduced by the first two authors in \cite{morris2026uniform}. Of particular interest to us here, it was shown that $d_{\mathcal{H},K}({A}_\eps,A_0) \to 0$ as $\eps\to 0$ given a unique Bayes classifier and under additional non-degeneracy and regularity assumptions that $d_{\mathcal{H},K}({A}_\eps,A_0) \lesssim \eps^{1/(d+2)}$. By comparison, our main result \eqref{eqn:introMain} shows that a \emph{dimension-independent} estimate holds for the geometric error, making the new optimal estimate meaningful in the context of high-dimensional statistics. The optimality of \eqref{eqn:introMain} is formally justified in Examples \ref{ex:lowReg} and \ref{ex:highReg} and by \cite{trillos2022adversarial}, where it was shown that $\eps\mapsto A_\eps$ is described in terms of a curvature flow.

We briefly remark on related literature, though we surely provide an incomplete review which mostly focuses on the theoretical and mathematical aspects of the problem. Existence of adversarial classifiers has also been proven by Awasthi et al.\ \cite{Awasthi2021OnTE} for a \emph{closed-ball} version of the optimization problem. The open-ball model we consider has the advantage of substantially simplifying questions regarding measurability of classifiers. Frank and Niles-Weed \cite{frankNW2024minmax} developed an existence theory for adversarial training with surrogate losses, and consistency of such models was studied by Awasthi et al.\ \cite{Awasthi2021calibration}.
 Building on this, Frank has identified natural uniqueness conditions for adversarial training in specific contexts \cite{frank2026notion,frank2026adversarial}. It has also been fruitful to view adversarial training as a distributionally robust optimization problem, opening the door to analysis via optimal transport. From this perspective, Pydi and Jog \cite{pydiJogOT} proved existence of (closed-ball) adversarial classifiers, and Garc{\' i}a Trillos et al. \cite{trillos2023existence,trillosJacobsKimMOT1} proved existence of multiclass classifiers and a reformulation of the problem in terms of multimarginal optimal transport. 
We are optimistic that our low (Lipschitz) regularity estimate \eqref{eqn:introMain}  will help quantify accuracy in the multiclass case, where the geometry must account for junctions where multiple classes meet.

\noindent \textbf{Outline of the strategy.} The intuition behind our arguments is contained in the proof of \eqref{eqn:introMain} for an $\ell^2$-adversary under \emph{assumed} regularity on the boundary of the Bayes and adversarial classifiers, see Section~\ref{sec:heuristic}. 
For this, we use the first-order necessary conditions for adversarial classifiers introduced in \cite{trillos2022adversarial} to deduce that all points in the boundary $\partial A_\eps$ must be sufficiently close to $\partial A_0$.
The proof is conditional as it is not known whether we can a priori assume any regularity for $A_\eps$; the presence of non-smooth $\sigma$ determining the adversary further complicates the matter. To extend the argument beyond these assumptions, we introduce a flexible method for constructing energetic competitors by removing a ball from a given classifier, cf.\ Lemma \ref{lem:basComp}. For minimizers, these energetic competitors are used to construct \emph{geometric barriers}, preventing $\partial A_\eps$ from moving too far away from $\partial A_0$. The barriers depend on both $\eps$ and the regularity of $A_0$, thereby leading to control by $\sqrt{\eps}$ or $\eps$ depending on whether $A_0$ has low (Lipschitz) or high ($C^{1,1}$) regularity, respectively. The construction of barriers is highly adaptable, and with high probability, we can prove that the geometric estimates hold in the finite data setting assuming the number of sampled points $n$ satisfies $n\gtrsim \eps^{-Cd}$, for a constant $C\in (3/2,3)$ independent of dimension. In the case $C=2$, this is roughly the number of samples needed to land within any ball of size $\eps^2$ in the domain. A fundamental challenge with finite data is quantifying geometric errors due to differences between the empirical distribution and the infinite data distribution. Using the results of Garc\'ia Trillos and Slep\v{c}ev \cite{trillos2015rate} to control the $\infty$-Wasserstein distance between these two distributions, we introduce energetic competitors for the finite data classifiers in Lemma \ref{lem:basCompEmp}. Due to the probabilistic setting, the competitor must be carefully `initialized' by an averaging argument that controls the interaction between distributional uncertainty and irregular geometry.

\noindent \textbf{Organization.} In Section \ref{sec:mainResults}, we precisely formulate the adversarial training problem and state the main results. The heuristic proof for convergence based on the first variation is given in Section \ref{sec:heuristic}. A brief review on adversarial training in metric spaces is contained in Section \ref{sec:prelim}. In Sections \ref{sec:main} and \ref{sec:emp}, we prove our main results in the continuum (infinite data) and empirical (finite data) settings, respectively.

\noindent \textbf{Notation.} We use $C>0$ and $c>0$ to denote constants that may change from line to line but are always independent of $\eps$. We indicate dependence via $C = C(\sigma)$, which says that $C$ depends on $\sigma$. For $x\in \R^d$, we often write $x = (x',x_d)\in \R^{d-1}\times \R.$ 
Let $\sigma:\R^d \to [0,\infty)$ be a norm on Euclidean space and $d(x,y) := \sigma(x-y)$ be the metric defined by $\sigma.$
Any object implicitly or explicitly depending on a norm is taken to be with respect to $\sigma$ unless stated otherwise, e.g., the ball $B(0,1) = \{x\in \R^d : d(x,0) = \sigma(x)<1\}$.
For a set $A\subset \R^d$, we define the (anisotropic) distance to the set as $d(x,A) :=\inf_{y\in A}d(x,y)$. For two sets $A$ and $B$, we write $d(A,B) : = \inf_{x\in A}d(x,B)$. Similarly, the signed distance to $A$ is defined as ${\rm sdist}(x,A) = d(x,A) - d(x,A^c).$ We use $|\cdot|_2$ to denote the standard Euclidean norm.

\section*{AI use declaration}

All mathematical ideas and proofs in the body of the paper were conceptualized and written by the authors.  ChatGPT 5.6 was useful in simplifying arguments and was used to lift Theorem \ref{thm:mainEmpirical} to non-smooth norms $\sigma$. In particular, we previously had an averaging argument applicable to smooth norms $\sigma$, where one could boil down estimates for fattened-perimeters to a graphical setting. Now we use Lemma \ref{prop:geometricIneq}, which circumvents graphical assumptions. The proof of Lemma \ref{prop:geometricIneq} was first provided by ChatGPT, though carefully written by the authors. We take full responsibility for the validity of all arguments presented here. Figure \ref{fig:avgArg} was made with the aid of AI too.

\section*{Acknowledgments}

R. Murray acknowledges partial support from NSF DMS 2307971 and a Simons travel grant. K. Stinson was supported by funding from the NSF RTG grant DMS-2136198. The authors thank Leon Bungert for many fruitful discussions relating to this work. Part of this work was conducted during a visit to the Casa Matem\'atica Oaxaca; R. Murray and K. Stinson are grateful for the hospitality during their visit.

\section{Main results} \label{sec:mainResults}

We divide our results into two settings: continuum and empirical. First, we consider at the adversarial training problem in the infinite data limit, where the energetic cost is in terms of the underlying data distribution $\mu$. We prove that the gap between the adversarial classifier and the Bayes classifier is controlled directly by robustness, via the parameter $\eps.$ Second, we introduce the empirical adversarial training problem and control the distance of the empirical minimizer to the Bayes classifier by $\eps$ and the number of sampled data points. 

\subsection{Convergence for the continuum adversarial training problem}
Let $\sigma:\R^d \to [0,\infty)$ be a norm on Euclidean space $\R^d$, $d(\cdot,\cdot)$ its associated metric, and $\mathcal{B}(\R^d)$ the collection of Borel measurable sets in $\R^d$.
Binary classifiers in adversarial training are determined by the optimization problem
\begin{equation}\label{eqn:advProb}
\min_{A\subset \mathcal{B}(\R^d)}\mathbb{E}_{(x,y)\sim \mu}\Big[\sup_{\tilde x \in {B(x,\eps)}}|\chi_{A}(\tilde x) - y|\Big] \equiv \min_{A\subset \mathcal{B}(\R^d)}\rho_0 (A^\eps) + \rho_1((A^c)^\eps),
\end{equation}
where $\mu$ is a probability measure on $\R^d\times\{0,1\}$, 
\begin{equation}\label{eqn:conditionalDensity}
\rho_i(\mathcal{O}) : = \mu(\mathcal{O}\times \{i\})
\end{equation} is the (unnormalized) conditional density, and $A^\eps : = \{x\in \R^d : d(x,A)<\eps\}$ is the $\eps$-dilation.
 We denote by $A_\eps$ a given minimizer of the adversarial classification problem \eqref{eqn:advProb}, which exists by Theorem 3.18 in \cite{bungert2023geom} (see also Section \ref{sec:prelim}). For convenience we write the adversarial risk above as
$$ \mathcal{R}_{\mu,\eps}(A) : = \rho_0 (A^\eps) + \rho_1((A^c)^\eps).$$ 

Convergence rates of $A_\eps$ towards the Bayes classifier rely on two basic ingredients: non-degeneracy of the densities and regularity of the Bayes classifier. Non-degeneracy ensures $A_0$ is the unique Bayes classifier and formally that this uniqueness is stable at first order.
When no confusion arises, we interchangeably use $\rho_i$ for the Lebesgue-density of the measure $\rho_i$ defined in \eqref{eqn:conditionalDensity}, i.e., $\rho_i(\mathcal{O}) = \int_{\mathcal{O}}\rho_i \, dx$.

\begin{definition}[Bayes classifier and non-degeneracy] \label{def:nondegen}
Letting $A_0 : = \{x: \rho_1 > \rho_0\}$ be the Bayes classifier, we say that the densities $\rho_i$ are {non-degenerate} if for any compact $K\subset\R^d$ there is $\zeta = \zeta(K) >0$ and a constant $c = c(K)>0$ so that $\rho_i\in C(\R^d)\cap C^{0,1}(\{d(x,\partial A_0)<\zeta\}\cap K)$, $\rho_0+\rho_1 > 0$, and 
\begin{equation}
\frac{\rho_0(x) - \rho_1(x)}{{\rm sdist}(x,A_0)}>c \qquad \text{for all $\quad x\in \{0<d(x,\partial A_0)<\zeta\} \cap K$.}\label{eqn:nondegen}
\end{equation}
\end{definition}

We emphasize that when $\partial A_0$ is locally $C^{1,1}$-regular (in the typical Euclidean sense), then \eqref{eqn:nondegen} follows from the simple non-degeneracy condition $\rho_i\in C^{1}(\R^d)$ and $\nabla (\rho_1 - \rho_0)\neq 0$ on $\{x: \rho_1 = \rho_0\}$.

Universally, we assume that $A_0$ is an open set with locally Lipschitz boundary, which, as we will show, is sufficient to recover an $O(\sqrt{\eps})$ control on geometric accuracy. To obtain the improved $O(\eps)$ estimate, we also require that the Bayes classifier $A_0$ is smooth with respect to the norm $\sigma$, referred to as \emph{$\sigma$-regular}.
If for instance $\sigma = |\cdot|_2$, the Euclidean norm, a $\sigma$-regular set has $C^{1,1}$-regular boundary and the $\sigma$-normal vector is the Euclidean normal vector. Similarly, in the case $\sigma$ belongs to $C^2(\R^d\setminus \{0\})$, $\sigma$-regularity coincides with standard $C^{1,1}$-regularity for $\partial A_0$ and the $\sigma$-normal vector is the Cahn--Hoffman vector. In the case that $\sigma$ is crystalline, e.g., $\sigma(x) = |x|_1$ or $\sigma(x) = |x|_\infty$, this allows for certain polygonal geometries for $A_0$. We refer to \cite{Bellettini2004} for examples of regular sets with respect to crystalline norms. One elementary example of a $\sigma$-regular set is the anisotropic ball $B(0,r)$. We state the precise definition of a $\sigma$-regular set, cf. Definition 3.1 in \cite{bellettini2001crystalline}.

\begin{definition}[$\sigma$-regular sets] \label{def:sigmaReg}
We say that $A_0$ is a locally $\sigma$-regular set if for any compact $K\subset\R^d$ there is a Lipschitz regular vector field $\nu:\partial A_0\cap K\to \R^d$ such that for $\mathcal{H}^{d-1}$-a.e. $x\in \partial A_0\cap K$, 
$$\nu (x) \in \left\{\xi : \left\langle\xi , \frac{\nu_{\rm Eucl}(x)}{\sigma^\circ(\nu_{\rm Eucl}(x))}\right\rangle = 1 = \sigma(\xi)^2\right\},$$
where $\nu_{\rm Eucl}(x)$ is the standard Euclidean outer normal of $A_0$ and $\sigma^\circ$ is the dual norm.
\end{definition}
We emphasize to the reader that the technical advantage of $\sigma$-regularity is that a tubular neighborhood of $\partial A_0$ may be defined in terms of tangential-normal coordinates (see Proposition \ref{prop:sigmaReg}).

Assuming the Bayes classifer is non-degenerate and regular, we show that any minimizer $A_\eps$ of \eqref{eqn:advProb} is (locally) $L^\infty$-close to $A_0$ and our estimate provides the optimal \emph{dimension-independent} scaling in $\eps$ determined by the underlying regularity of the classifier. For a given set $K\subset \R^d$, define
\begin{equation}\label{eqn:HausdorffDef}
d_{\mathcal{H},K} (A,B) : = \sup \left\{\sup_{x\in A\cap K} \, d(x, B) , \sup_{x\in B\cap K} \, d(x, A)\right\}, 
\end{equation}
which is simply the Hausdorff metric localized to $K$.

\begin{theorem}\label{thm:main}
Suppose that $A_0\subset \R^d$ is the Bayes classifier with locally Lipschitz boundary, the densities $\rho_i$ are non-degenerate (see Definition \ref{def:nondegen}), and $A_\eps$ is a minimizer of the adversarial classification problem \eqref{eqn:advProb} for $\eps \in (0,1]$. Given $K\subset \subset \R^d$ the following holds: 
\begin{enumerate}
\item \label{item:lowReg} \emph{(Low regularity).}
There is a constant $C_0 = C_0(\rho_i, \sigma, d, K)>0$ such that 
\begin{equation}\label{eqn:optGeneric}
d_{\mathcal{H},K}(A_\eps,A_0)\leq C_0 \sqrt{\eps}  \quad \quad \text{ for } \eps \in (0,1].
\end{equation}
\item \label{item:highReg} \emph{(High regularity).}
If in addition $A_0$ has locally $\sigma$-regular boundary {(see Definition \ref{def:sigmaReg})}, then there is a constant $C_1 = C_1(\rho_i, \sigma, d, K)>0$ such that 
\begin{equation}\label{eqn:optRegular}
d_{\mathcal{H},K}(A_\eps,A_0)\leq C_1 \eps  \quad \quad \text{ for } \eps \in (0,1].
\end{equation}
\end{enumerate}
\end{theorem}

The proof of Theorem \ref{thm:main} is contained in Section \ref{sec:main}. For intuition, in Section \ref{sec:heuristic} we provide a simpler proof based on the first variation of the energy in the case that the boundary of $A_\eps$ is assumed to be regular and $\sigma = |\cdot|_2$, the Euclidean norm.

We include a couple remarks on the theorem.

\begin{remark}[Optimality]
Theorem \ref{thm:main} provides the best expected control on accuracy via $\eps$ given the underlying regularity of $A_0$. In Example \ref{ex:lowReg} we provide a heuristic calculation for a Lipschitz surface when $\sigma = |\cdot|_2$ showing $d_{\mathcal{H},K}(A_\eps,A_0) \geq  c\sqrt{\eps}$ over a natural class of energetic competitors. Similarly, in Example \ref{ex:highReg}, we show that $O(\eps)$ is optimal in the high regularity setting.
Another way to look at the high regularity setting is through \cite{ngt2022NecCond}, which shows that $\eps\mapsto A_\eps$ is formally given by a nonlocal curvature flow, which agrees with the fact that \eqref{eqn:optRegular} is optimal.
\end{remark}

\begin{remark}[Dependence on dimension]
The dependence of constants in Theorem \ref{thm:main} on the dimension is typically polynomial given fixed information on the densities $\rho_0$ and $\rho_1$. In the high regularity setting, letting $\zeta>0$ and $c>0$ be small so as to satisfy both Definitions \ref{def:nondegen} and \ref{def:sigmaReg} on the dilation $K^1$ and defining $\gamma : = \inf_{x\in K^1 \cap \{d(x,\partial A_0)>\zeta/4\} } |\rho_0(x) - \rho_1(x)|>0,$
we have that
$$C_1 = \tilde C_1 \, d,$$
for a constant $\tilde C_1 = \tilde C_1(\zeta,c,\gamma, {\rm diam}(K), \sigma, \omega_{\rho_i,K^1})>0$ where $\omega_{\rho_i,K^1}$ is the modulus of continuity for $\rho_i$ in $K^1$.
A similar conclusion holds in the low regularity setting, though the constant also depends on the equivalence of $\sigma$ and $|\cdot|_2.$
For the interested reader, we make many of the constant dependencies transparent in the proof of Theorem \ref{thm:main}.
\end{remark}

\begin{remark}
We take a slight deviation from \cite{morris2026uniform} and do not control $d_\mathcal{H}(A_0\cap K, A_\eps\cap K)$. To see that such an adjustment is necessary (though benign), it is possible $A_0\subset K^c$ and $A_\eps \cap K \neq \emptyset$, meaning that $d_\mathcal{H}(A_0\cap K, A_\eps\cap K) = +\infty$.
\end{remark}

\subsection{Convergence for the empirical adversarial training problem}\label{subsec:mainEmp}

We now discuss convergence rates for the empirical adversarial training problem given data sampled from $\mu$. In the high regularity setting where we hope for an $O(\eps)$ control, we expect to require enough samples to resolve the boundary layer of width $\eps$. As we should have this resolution everywhere along the boundary $\partial A_\eps$ (an unknown), this requires sufficient data for every $\eps$-ball in the domain. A direct estimate for this heuristic says the number of samples $n$ should scale like $n \gtrsim \log(1/\eps)/\eps^d$ so that failure occurs with probability less than $\eps.$ While we do not recover this exact scaling, we require a similar number of samples---at most $1/\eps^{3d}$---to translate the convergence rate from Theorem \ref{thm:main} to the empirical adversarial training problem. 

We simplify the setting by assuming that the (unnormalized) conditional densities $\rho_i$ defined in \eqref{eqn:conditionalDensity} are supported in $\Omega$, a bounded connected set with Lipschitz boundary, and that there is $c_\rho>0$ so that $\rho_i \geq c_\rho$ in $\Omega$ for $i=0,1.$
If these conditions hold and Definition \ref{def:nondegen} holds with $\R^d$ replaced by $\Omega$, we say \emph{the densities $\rho_i$ are non-degenerate in $\Omega$.}

We introduce the empirical problem.
  Let $\{(x_j,y_j)\}_{j\in \mathbb{N}} \subset \Omega$ be sampled from $\mu$ and define the \emph{discrete domain} $$X_n : = \{x_j: j\leq n\}.$$ For $n\in \N$, using the samples $\{(x_j,y_j)\}_{j=1}^n,$ we define the empirical measures
\begin{equation}\label{eqn:empMeas}
 \nu_0 : = \frac{1}{n}\sum_{{\{j\leq n:\ y_j  =0\}}}\delta_{x_j} \quad \text{ and } \quad \nu_1 : = \frac{1}{n}\sum_{\{j\leq n:\ y_j  =1\}}\delta_{x_j}.
\end{equation}
We define the empirical adversarial classification problem as follows: Find $\mathcal{A} \subset X_n \subset\Omega$ minimizing the risk functional
\begin{equation}\label{eqn:advProbEmp}
\mathcal{R}_{\nu,\eps}(\mathcal{A}) : =\nu_0(\mathcal{A}^\eps) + \nu_1((\mathcal{A}^c)^\eps),
\end{equation}
where $\mathcal{A}^c : = X_n\setminus \mathcal{A}$ is the complement relative to the discrete domain and the dilations are defined by the anistropic distance $\sigma$ as before. Following the work of \cite{bungert2023geom}, minimizers of this problem exist and we denote an empirical minimizer by $\mathcal{A}_{\eps} = \mathcal{A}_{\eps}(n)$ (see Section \ref{sec:prelim}). 

Following the classical approach for ensuring dense sampling with respect to a specified length scale, we define the \emph{distributional error} 
\begin{equation}\label{eqn:distrError}
 \eta = \eta(n) : = \begin{cases}
\frac{\log(n)^{3/4}}{n^{1/2}} & \text{ if }d=2,\\
\frac{\log(n)^{1/d}}{n^{1/d}} & \text{ if }d>2,
\end{cases}
\end{equation}
which we use to quantify the distance between $\rho_i$ and $\nu_i$, {cf.\ \cite[Theorem 1.1]{trillos2015rate} where $\eta$ appears.}
Recalling $d_{\mathcal{H},K}$ from \eqref{eqn:HausdorffDef}, we prove the following in Section \ref{sec:emp}.

\begin{theorem}\label{thm:mainEmpirical}
Suppose that $A_0\subset \Omega$ is the Bayes classifier with locally Lipschitz boundary in $\Omega$, the densities $\rho_i$ are non-degenerate in $\Omega$ (see above \eqref{eqn:advProbEmp}), and $\mathcal{A}_\eps = \mathcal{A}_\eps(n)$ is a minimizer of the empirical adversarial classification problem \eqref{eqn:advProbEmp} for $\eps \in (0,1]$ and $n\in \N$. Given $K\subset \subset \Omega$ and $\alpha>2$, the following holds:
\begin{enumerate}
\item \label{item:lowRegEmp} \emph{(Low regularity).} There is a constant $C_0 >0$ and ${c_0 = c_0(\alpha)}>0$ such that if $\eta(n) < c_0\eps^{3/2}$ then
\begin{equation}\nonumber 
d_{\mathcal{H},K}(\mathcal{A}_\eps,A_0)\leq C_0 \sqrt{\eps}  \quad \quad \text{ for } \eps \in (0,1] \quad \text{ with probability } 1-O(n^{-\alpha/2}).
\end{equation}
\item \label{item:highRegEmp} \emph{(High regularity).}
Supposing in addition $A_0$ has locally $\sigma$-regular boundary in $\Omega$ {(see Definition \ref{def:sigmaReg})} and $\kappa>0$, there is a constant $C_1>0$ and ${c_\kappa = c_\kappa(\alpha)} > 0$ such that if $\eta(n) < c_\kappa\eps^{2+\kappa}$ then
\begin{equation}\nonumber 
d_{\mathcal{H},K}(\mathcal{A}_\eps,A_0)\leq C_1 \eps  \quad \quad \text{ for } \eps \in (0,1] \quad \text{ with probability } 1-O(n^{-\alpha/2}).
\end{equation}
\end{enumerate}
\end{theorem}

As this argument is substantially more involved than the continuum/infinite data setting and depends on the equivalence of $\sigma$ and $|\cdot|_2$ even in the high regularity setting, we do not provide as close of an accounting of constants and their dependencies, though the general dependence is as expected: meaning $C_1 = C_1(\rho_i, \sigma, d, K, \kappa,\alpha)>0.$

\begin{remark}[Number of samples]
Following the heuristic discussion from the beginning of this subsection, we would hope to require $\eta\lesssim \eps$ in the high regularity setting. Ignoring $\kappa$, we pay the price of an extra factor of $\eps$ with $\eta \lesssim \eps^2$, meaning we {have sampled well in} all balls of radius $\eps^2$ with high probability. The same $\eps$ decrease in $\eta$ is found in the low regularity setting.
The reason for the increase in samples is relatively technical: the proof of Theorem \ref{thm:mainEmpirical} relies on delicate averaging arguments to control Minkowski-type perimeter terms arising due to the gap between the empirical and continuum data distributions (see Lemma \ref{lem:basCompEmp}). Importantly, however, this shortcoming is not exacerbated in higher dimensions. 
\end{remark}

\begin{remark}[The domain $\Omega$]
Using densities that are non-degenerate in $\Omega$ allows us to take advantage of results in \cite{trillos2015rate} giving Wasserstein convergence rates for empirical distributions to their continuum counterparts. It is likely that Theorem \ref{thm:mainEmpirical} holds in $\R^d$ (analogous to the statement of Theorem \ref{thm:main}) with more work spent localizing the Wasserstein convergence to regions where the densities are bounded away from $0$. However, as this would distract from the essential difficulties, we do not pursue this here.
\end{remark}

\section{Heuristic argument: Proof in smooth case}\label{sec:heuristic}

We begin with a proof of convergence rates in the case that $A_\eps$ is smooth. While the regularity of $A_\eps$ is not well-understood in general, the method of proof illuminates the basic argument in a simpler setting than Theorem \ref{thm:main}. As this result is not a logical prerequisite for our general proof, we avoid technical details where they hinder intuition.

\begin{proof}[Heuristic argument for Theorem \ref{thm:main}\eqref{item:highReg}]
We suppose $\sigma = |\cdot|_2$ is the Euclidean norm, that $A_\e$ is a minimizer of the adversarial training problem, and that $A_\e$ has compact $C^2$-boundary. Let $A_0$ have compact $C^2$-boundary and be a non-degenerate Bayes classifier as in Definition \ref{def:nondegen} such that $\nabla (\rho_1 - \rho_0)\neq 0$ on $\partial A_0$ and $\rho_0,\rho_1\in C^1(\R^d)$. Under these hypothesis, we sketch a proof of Theorem \ref{thm:main} based on the first variation. We point out that by convention the curvature of a ball is positive.

The minimizer $A_\eps$ satisfies the first order necessary conditions identified in \cite{ngt2022NecCond}, Proposition 6.2, given by
    \begin{equation}\label{eqn:necessary-condition-NGTRM}
         \rho_1(x -\e \nu(x)) \prod_{i=1}^{d-1} |1 - \kappa_i(x) \e| - \rho_0(x +\e \nu(x)) \prod_{i=1}^{d-1} |1 +  \kappa_i(x) \e| = 0 \qquad \text{ for }x\in \partial A_\eps \quad \text{ for }\eps\ll 1.   
    \end{equation}
    Here $\nu$ is the outer normal of the boundary $\partial A_\e$ and $\kappa_i$ is the $i$-th curvature of $\partial A_\eps$. 
    We \emph{assume} $|\kappa_i|\leq 1/\eps$ and then rearrange the necessary conditions to obtain
    \begin{equation}\label{eqn:rearrange}
        \frac{\rho_1(x-\eps\nu)}{\rho_0(x+\e\nu)} = \prod_{i=1}^d
        \frac{1+\kappa_i(x)\e}{1-\kappa_i(x)\e}.
    \end{equation}

Now, let us consider the point $x_0 \in \partial A_\e$ which is farthest from $A_0$, namely
\[
x_0 \in {\rm argmax}_{x \in A_\e} d(x,A_0).
\]
We note that by the suboptimal convergence rates in \cite{morris2026uniform}, we may assume that $d(x_0,A_0) \ll 1$ for $\eps$ sufficiently small.
Since $x_0$ is the farthest point of $A_\eps$ from $A_0$, that we immediately have $\kappa_i(x_0) > -C_0$ for some $C_0 = C_0(A_0)>0$ only depending on $A_0$ (and not on $\eps$!) and that $d(x_0,A_0) \ll 1$.  
(Control on the curvatures is analogous to having $f'' \leq 0$ at a maximum of $f\in C^2(\R)$, where the condition $d(x_0,A_0) \ll 1$ ensures $\partial A_\eps$ can be expressed as a graph over $\partial A_0$ locally near $x_0$.)
The curvature lower bound gives that $$\frac{1+\kappa_i(x_0)\e}{1-\kappa_i(x_0)\e}\geq \frac{1 -C_0\e}{1+C_0\eps}$$ from which one deduces
\begin{equation}\label{eqn:rearrange2}
\prod_{i=1}^d
        \frac{1+\kappa_i(x_0)\e}{1-\kappa_i(x_0)\e} \geq 1-C\eps,
\end{equation}
        for a constant $C = C(A_0,d)>0$.
        Using the regularity of $\rho_i$, the lower bound on $\rho_0+\rho_1$, along with the fact that $d(x_0,A) \ll 1$, we can bound
        $$\frac{\rho_1(x_0-\eps\nu)}{\rho_0(x_0+\e\nu)} \leq  \frac{\rho_1(x_0)}{\rho_0(x_0)} +C\eps,$$
        for $C = C(\rho_0,\rho_1)>0$. We then combine \eqref{eqn:rearrange} and \eqref{eqn:rearrange2} to find
        $$\frac{\rho_1(x_0)}{\rho_0(x_0)} \geq 1-C\eps$$
        for a constant $C>0$ independent of $\eps.$
        On the other hand, by the non-degeneracy condition \eqref{eqn:nondegen}, we have
        $$\frac{\rho_1(x_0)}{\rho_0(x_0)} = 1+ \frac{\rho_1(x_0)-\rho_0(x_0)}{\rho_0(x_0)} \leq 1-cd(x_0,A_0). $$ 
        Altogether the previous two displays imply $d(x_0,A_0) < (C/c)\eps,$ or equivalently
        $\sup_{x\in \partial A_\eps}d(x,A_0) < C\eps$ for a constant $C>0$ independent of $\eps.$ 
        Combined with the analogous argument for the farthest point in $A_\eps^c$ from $A_0^c$, we conclude that $\partial A_\eps \subset (\partial A_0)^{C\eps}$, which then implies the desired control
        $d_{\mathcal{H}}(A_\eps,A_0)<C\eps $ on the Hausdorff metric.
\end{proof}

\begin{remark}
In the complete proof of Theorem \ref{thm:main}\eqref{item:highReg}, non-degeneracy is effectively used in the same way. The essential challenge is how to circumvent the assumed regularity of the $\partial A_\eps$. The key idea is that, for any minimizer, the points that are farthest away from $A_0$ ($x_0$ above) come with mild, one-sided regularity properties which we can leverage to construct energetic competitors.
\end{remark}

\section{Preliminaries on adversarial training}\label{sec:prelim}

For the reader's convenience, we briefly recall some notation and results for the adversarial classification problem that we use throughout the paper. 
Given $A\subset X$, where $(X,d)$ is a metric space, the $\eps$-dilation of $A$ in $X$ is $$A^\eps : = \{y\in X : d(y,A)<\eps\}.$$ The $\eps$-erosion of $A$ in $X$ is $$A^{-\eps} : = \{y\in X : d(y,X\setminus A)\geq \eps\}.$$ The opening and {closing} of $A$ in $X$ is given by ${\rm op}_\eps(A) : = (A^{-\eps})^\eps$ and ${\rm cl}_\eps(A) : = (A^{\eps})^{-\eps}$, respectively. 
The {closing} and opening satisfy simple subset relations, which we record in the following proposition.

\begin{proposition}[\cite{haralick1992computer}] \label{prop:openCloseBasicProperties}
In a metric space $(X,d)$ the following hold:
\begin{align*}
&\bullet \ {\rm op}_\eps (A) \subset A, \qquad &&\bullet \ {\rm op}_\eps(A)^{-\eps} =A^{-\eps},  \\
&\bullet \ {\rm cl}_\eps (A) \supset A, \qquad && \bullet \  {\rm cl}_\eps(A)^\eps = A^\eps, \\
&\bullet \ A^{-\eps} = ((A^c)^\eps)^c, \qquad && \bullet \ {\rm op}_\eps(A^c) = {\rm cl}_\eps(A)^c.
\end{align*}
\end{proposition}

The metric spaces we are interested in are the Euclidean space $\R^d$ or finitely many samples $X_n\subset\Omega$, both equipped with the metric determined by $\sigma.$ Within the proof of Theorem \ref{thm:mainEmpirical}---as we must connect discrete objects to continuum counterparts---we explicitly state whether to use dilation/erosion in $X_n$ or $\R^d.$ Typically, we use calligraphic writing to denote a set $\mathcal{A}\subset X_n$ to remind the reader that erosions/dilations/complements should be done within $X_n$. For dilations however it usually does not matter because, for instance, $\nu_0(\mathcal{A}^\eps)$ is the same regardless of using the dilation in $X_n$ or $\R^d$.

For $\eps>0$, Theorem 3.18 \cite{bungert2023geom} shows both problems \eqref{eqn:advProb} and \eqref{eqn:advProbEmp} have solutions denoted by $A_\eps$ and $\mathcal{A}_\eps$. In fact, we may recover further solutions by taking the {closing} and opening. 

\begin{proposition}\label{prop:openclose}
Let $A_\eps$ be a minimizer of \eqref{eqn:advProb} with $\mu$ a probability measure on $\R^d\times \{0,1\}.$ Then for any measurable set $U$ it holds that
$$\rho_0 ( {\rm op}_\eps (A_\eps)^{\eps}\cap U) = \rho_0 ( A_\eps^{\eps}\cap U) \quad \text{ and }\quad \rho_1 ( {\rm op}_\eps (A_\eps)^{-\eps}\cap U) = \rho_1 ( A_\eps^{-\eps}\cap U),$$ and likewise with the opening replaced by the {closing}. In particular, ${\rm op}_\eps(A_\eps)$ and ${\rm cl}_\eps(A_\eps)$ minimize \eqref{eqn:advProb}. The analogous result holds for a minimizer $\mathcal{A}_\eps$ of \eqref{eqn:advProbEmp} with erosions/dilations defined relative to $X_n$ and $\nu_i$ in place of $\rho_i$.
\end{proposition}

This proposition follows from Lemmas 3.28 and Corollary 3.29 in \cite{bungert2023geom}, where it is shown in the case of a metric space $(X,d)$ with $U = X$. The relation can be localized since ${\rm op}_\eps (A_\eps)^{\pm \eps} \subset A_\eps^{\pm \eps}$, and similarly for the {closing}.

\section{The continuum problem}\label{sec:main}

We decompose the proof of Theorem \ref{thm:main} into its two parts. In both cases our approach is to construct an energetic competitor by removing a ball. With the competitor in hand, the general strategy for the main theorem is to choose as large of a ball as possible that stays on one side of $\partial A_0$ in the non-degenerate region, thus leading to a favorable energy improvement, and contradicting minimality of a given classifier.

\subsection{Energetic competitor}
We construct our basic energetic competitor in the following lemma. By choosing $x_*$ and $r$ depending on the circumstance, this construction can be directly applied in a variety of scenarios.

\begin{lemma}[Basic competitor]\label{lem:basComp}

Let $A \subset \R^d$ and $\eps>0$. For $x_* \in \R^d$ and $r> 0 $, we define 
$$\mathbb{A} : = A \setminus (\mathbb{B}^\eps) \quad \text{ with } \quad \mathbb{B} :  = B (x_*,r).$$ Then for
\begin{equation}\label{eqn:MJM}
\lambda : = \frac{r+2\eps}{r}, \qquad M (x) : = x_* + \lambda(x-x_*), \quad \text{ and }\quad  J_M : = \lambda^{d}
\end{equation}
we have
\begin{equation}\label{eqn:energyImprovement}\mathcal{R}_{\mu,\eps}(\mathbb{A}) + \int_{A^\eps \cap \mathbb{B}}\Big(\rho_0(x) - \rho_1(M(x)) J_M\Big) \,dx \leq \mathcal{R}_{\mu,\eps}({A}). 
\end{equation}
\end{lemma}

This lemma shows that we can improve the energy of $A$ by excising the ball $\mathbb{B}$ anytime the integral on the left-hand side of \eqref{eqn:energyImprovement} is over a domain with positive measure and has positive integrand.

\begin{proof}[Proof of Lemma \ref{lem:basComp}]
Since $\mathbb{A}\subset A$, we have $\mathbb{A}^\eps\subset A^\eps$ and $(A^c)^\eps\subset (\mathbb{A}^c)^\eps$. From this we rewrite
\begin{equation}\label{eqn:riskRewrite}
 \mathcal{R}_{\mu,\eps}({A}) -\mathcal{R}_{\mu,\eps}(\mathbb{A}) = \rho_0(A^\eps \setminus \mathbb{A}^\eps) - \rho_1((\mathbb{A}^c)^\eps \setminus (A^c)^\eps).
 \end{equation}
To estimate the right-hand side, we prove a useful set relation.
Note that for $N :  =M^{-1}$, it follows $N(\mathbb{B}^{2\eps}) = \mathbb{B}$ and for any $x\in \mathbb{B}^{2\eps}$ it holds that $d(x,N(x))  = (1-\lambda^{-1})d(x,x_*)<2\eps$. We claim that 
\begin{equation}\label{eqn:setInclus}
N((\mathbb{A}^c)^\eps \setminus (A^c)^\eps)\subset A^\eps \cap \mathbb{B}. 
\end{equation}
To see this, consider $x \in (\mathbb{A}^c)^\eps \setminus (A^c)^\eps.$ We have by definition of $\mathbb{A}$ that
 $\mathbb{A}^c = A^c \cup \mathbb{B}^\eps$ so that $x\in (\mathbb{A}^c)^\eps \setminus (A^c)^\eps \subset \mathbb{B}^{2\eps}$. For such $x$, by the definition of $N,$ $N(x)\in \mathbb{B}$. 
 Since $x\in ((A^c)^\eps)^c = A^{-\eps}$, we have that $$N(x)\in B(x,2\eps)\subset (A^{-\eps})^{2\eps} = {\rm op}_\eps(A)^\eps \subset A^\eps $$ where we have used ${\rm op}_\eps(A) \subset A$ by Proposition \ref{prop:openCloseBasicProperties}, leading to \eqref{eqn:setInclus}.

Applying a change of variables and \eqref{eqn:setInclus}, we have
\begin{equation}\label{eqn:CCooVV} \rho_1((\mathbb{A}^c)^\eps \setminus (A^c)^\eps)  = \int_{N((\mathbb{A}^c)^\eps \setminus (A^c)^\eps)}\rho_1(M(x)) J_M \,dx \leq \int_{A^\eps \cap\mathbb{B}} \rho_1(M(x)) J_M \, dx,
\end{equation}
where $J_M$ is the Jacobian of $M$ (as in the lemma's statement). Using that $\mathbb{A}^\eps \subset \mathbb{B}^c$ by definition, we have $A^\eps \cap \mathbb{B} \subset A^\eps \setminus \mathbb{A}^\eps,$ which when combined with \eqref{eqn:riskRewrite} and \eqref{eqn:CCooVV} concludes the lemma.
\end{proof}

\subsection{Convergence in the low regularity setting}

We turn to the proof of Theorem \ref{thm:main}\eqref{item:lowReg} in the case of low regularity.
We begin with the proof of a one-sided bound. We control how far outside of $A_0$ the erosion of the optimizer, i.e., $A^{-\eps}_\eps$, extends. Given the symmetry of the argument with respect to labels, it also gives a bound for how far inside $A_0$ the set $(A^\eps_\eps)^c$ pushes.

\begin{proposition}\label{prop:DboundSimple}
    Under the hypotheses of Theorem \ref{thm:main}\eqref{item:lowReg}, there is $C = C (\rho_i, \sigma, d, K)>0$ and $\eps_0 = \eps_0 (\rho_i,\sigma, d, K)>0$ such that $$\sup_{x \in A_\eps^{-\eps} \cap K} d(x,A_0) < C\sqrt{\eps} \quad \text{ for }\eps<\eps_0.$$
\end{proposition}

\begin{proof}
 Let $\zeta\in (0,1)$ be sufficiently small  so that Definition \ref{def:nondegen} holds on $\overline{K^1}$, namely an expanded version of $K$. Suppose, with the aim of establishing a contradiction, that 
 \begin{equation}\label{eqn:contraHyp}
 D : = \sup_{x \in A_\eps^{-\eps} \cap K} d(x,A_0) > C_0\sqrt{d\eps}
 \end{equation}
  for $C_0>0$ chosen below (independent of the dimension). Consider $x_* \in A_\eps^{-\eps} \cap K $ with $d(x_*,A_0) = D$, which exists as $A_\eps^{-\eps}$ and $A_0$ are closed and $K$ is compact.

 \emph{Step 1 (Generic competitor and estimates).}  Define
 $$\mathbb{A}_\eps : = A_\eps\setminus \mathbb{B}^\eps \quad \text{ with }\quad \mathbb{B} : =B(x_*,r) $$
 as in Lemma \ref{lem:basComp}: we will make an appropriate choice for $r$ dependent upon two cases below (see also Figure~\ref{fig:barrier-comparison} below).
 We have
 \begin{equation}\label{eqn:energyImprove1}
 \mathcal{R}_{\mu,\eps}(\mathbb{A}_\eps) + \int_{A_\eps^\eps \cap \mathbb{B}}\Big(\rho_0(x) - \rho_1(M(x)) J_M\Big) \,dx \leq \mathcal{R}_{\mu,\eps}({A}_\eps),
\end{equation}
where $M$ and $J_M$ are defined in \eqref{eqn:MJM}.
First, we note it holds that 
\begin{equation}\label{eqn:posVol}
B(x_*,2\eps) \subset A_\eps^\eps \cap \mathbb{B}
\end{equation} so long as $2\eps < r,$
so $A_\eps^\eps \cap \mathbb{B}$ has positive measure.
We estimate the integrand from below for $x\in \B$ with
\begin{align*}
\rho_0(x) - \rho_1(M(x))J_M &\geq \rho_0(x) - \rho_1(x) + \rho_1(x) -  \rho_1(M(x)) + \rho_1(M(x))(1-J_M) \\
&\geq \rho_0(x) - \rho_1(x) - \omega_{\rho_1,\mathbb{B}^{2\eps}}(2\eps) -  \|\rho_1\|_{L^\infty({K^1})}(J_M-1),
\end{align*}
where $\omega_{\rho_1,U}$ is the modulus of continuity of $\rho_1$ on a set $U$ and we use that $d(x,M(x)) < 2\eps$ in $\mathbb{B}$. We also require that $\eps_0 < 1/4$ and $r < 1/2$ so that $M(x)$ necessarily belongs to $K^1$.
To estimate the Jacobian, we linearize the exponential with
\begin{equation}\nonumber
J_M - 1  = \left(1 + \frac{2\eps}{r}\right)^d - 1 \leq \exp(2\eps d/r)  -1 \leq 4\frac{\eps d}{r}
\end{equation}
where the last inequality holds if $2\eps d/r \leq 1/2.$ 
Inserting this into our prior estimate we obtain
\begin{equation}\label{eqn:inteLB}
\rho_0(x) - \rho_1(M(x))J_M \geq \rho_0(x) - \rho_1(x) -  \omega_{\rho_1,\mathbb{B}^{2\eps}}(2\eps) - 4\|\rho_1\|_{L^\infty(K^1)} \frac{\eps d}{r}
\end{equation}
for $2\eps d/r \leq 1/2$.
We reduce to cases based upon on whether $x_*$ is close to or far from $\partial A_0$ to show that the right-hand side of \eqref{eqn:inteLB} is strictly positive. Once this holds, \eqref{eqn:posVol} and \eqref{eqn:energyImprove1} contradict the minimality of $A_\eps$ showing that \eqref{eqn:contraHyp} cannot hold.


\begin{figure}[t]
    \centering

    \begin{minipage}[t]{0.48\textwidth}
        \centering
        
            \includegraphics[scale=0.9]{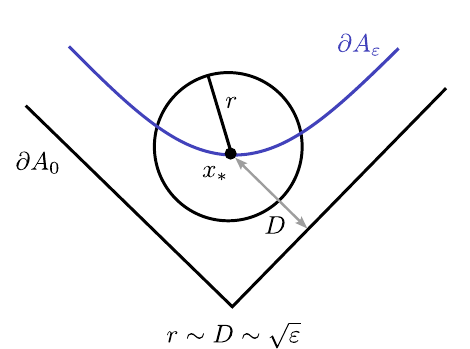}%
        
        \vspace{0.3em}

    \end{minipage}
    \hfill
    \begin{minipage}[t]{0.48\textwidth}
        \centering
        
            \includegraphics[scale=0.9]{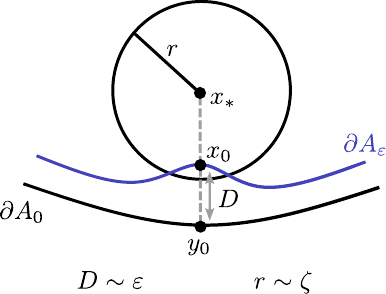}%
        
        \vspace{0.3em}

    \end{minipage}

    \caption{
    The geometric barrier constructions for low (left) and high (right) regularity settings, as in Propositions \ref{prop:DboundSimple} and \ref{prop:Dbound}, respectively.
    In the low regularity setting, the barrier radius $r$  is proportional to the distance $D$ between $A_0$ and $x^*\approx {\rm argmax}_{\partial A_\eps}d(x,A_0)$.
    By contrast, in the high regularity setting, we use normal coordinates to select
    a radius independent of $D = d(x_0,A_0)$.
    }
    \label{fig:barrier-comparison}
\end{figure}


\emph{Step 2 ($x_*$ far from $A_0$).}
Suppose that $d(x_*,A_0) \geq  \zeta/2$. Let $r = \zeta/4$ and define $$\gamma : = \inf_{x \in K^1 \cap \{d(y,A_0) \geq  \zeta/4\}}(\rho_0-\rho_1) > 0.$$ For $\eps < c (\omega_{\rho_1,K^1})$, we have $\omega_{\rho_1,\mathbb{B}^{2\eps}}(2\eps) \leq \omega_{\rho_1,K^1}(2\eps) < \gamma/2.$ Further, using that $r>\zeta/8$, we estimate \eqref{eqn:inteLB} from below by 
\begin{equation}\label{eqn:helpfuLB}
\rho_0(x) - \rho_1(x) -  \omega_{\rho_1,\mathbb{B}^{2\eps}}(2\eps) - 4\|\rho_1\|_{L^\infty(K^1)} \frac{\eps d}{r}\geq \gamma/2  - 32\|\rho_1\|_{L^\infty(K^1)} \frac{\eps d}{\zeta},
\end{equation}
 which is valid when $\eps \leq \zeta/(16d) $. Thus taking $\eps < \eps_0(\gamma,\zeta, \omega_{\rho_1,K^1},\|\rho_1\|_{L^{\infty}(K^1)})/d$ gives a positive lower bound to \eqref{eqn:helpfuLB} contradicting minimality. We conclude that
\begin{equation}\label{eqn:initSmall}
d(x_*,A_0) \leq  \zeta/2.
\end{equation}

\emph{Step 3 ($x_*$ close to $A_0$).}
With \eqref{eqn:initSmall} in hand, we may use non-degeneracy \eqref{eqn:nondegen} to estimate \eqref{eqn:inteLB} from below. Let $r = D/2$; this is consistent with the inequality $2\eps d/r \leq 1/2$ for $\eps_0$ sufficiently small. Now, by letting $\eps_0(\zeta)$ be small enough, we have $\mathbb{B}^{2\eps} \subset K^1 \cap \{0<d(y,A_0) < \zeta\}$, we have
\begin{equation}\label{eqn:helpfuLB1}
\begin{aligned}
\rho_0(x) - \rho_1(x) -  \omega_{\rho_1,\mathbb{B}^{2\eps}}(2\eps) - C\frac{\eps d}{r} & \geq c\, d(x_*,A_0)  - C_{\rm Lip}\eps - C\frac{\eps d}{D} \\
 &\geq c C_0\sqrt{d\eps} - C_{\rm Lip}\eps  - \frac{C}{C_0} \sqrt{\eps d}
\end{aligned}
\end{equation}
where we have used \eqref{eqn:contraHyp} and $C_{\rm Lip}$ just depends on $\|\nabla \rho_1\|_{L^\infty(K^1 \cap \{0<d(y,A_0) < \zeta\})}$.
Choosing $C_0$ large enough so that $\tilde C :=c C_0 - \frac{C}{C_0} >0$ and then considering $\sqrt{\eps} < \tilde C/(2C_{\rm Lip})\sqrt{d} $ gives positivity of the expression in \eqref{eqn:inteLB}, contradicting minimality of $A_\eps$ and concluding the proof.
\end{proof}

We remark that the above proof does not actually rely on the existence of a locally Lipschitz boundary for $\partial A_0$.
At this point, up to ignoring 
{the opening of $A_\eps$}
(which can be dealt with as in the proof of Theorem \ref{thm:main}\eqref{item:lowReg} below), Proposition \ref{prop:DboundSimple} can be used to show the subset relation 
\begin{equation}\label{eqn:noreg}
 A_0^{-C\sqrt{\eps}} \subset  A_\eps \cap K \subset A_0^{C\sqrt{\eps}},
 \end{equation}
 already providing meaningful information on the geometry of minimizers with \emph{no regularity assumptions} on $\partial A_0.$ It is only to obtain the symmetric relation with $A_0$ and $A_\eps$ swapped that we require $A_0$ be locally Lipschitz. At a high level, Lipschitz regularity gives a direction to move that is guaranteed to be \emph{away} from $\partial A_0$ (in the graphical setting this is in the $x_d$-th direction). We formalize this in the following lemma.

\begin{lemma}\label{lem:pushAway}
Suppose $A_0\subset \R^d$ has locally Lipschitz boundary and $K\subset\subset \R^d$. There are constants $\zeta, \ C, \ c > 0$ such that for any $x\in A_0\cap K$ and $t\in (0,\zeta)$ such that $d(x,A_0^c) < t,$
there is $y\in A_0$ such that 
\begin{equation}\label{eqn:pushAway}
d(y,x)\leq C t  \quad \text{ and }\quad d(y,A_0^c) \geq c t.
\end{equation}
\end{lemma}

\begin{proof}
Suppose $A_0 : = {\rm subgraph}(f)$ for Lipschitz $f:\R^{d-1}\to \R$ with (Euclidean) Lipschitz constant $L.$ We take $\tilde y = (\tilde y',f(\tilde y'))\in {\rm graph}(f)$ such that $d(x,\tilde y) = d(x,A_0^c)$. Define $y : = (\tilde y',f(\tilde y')-t)$. Define the cone $\mathcal{C} : = \{z = (z',z_d):z_d > - L|z'|_2\}$. Using that $A_0^c \subset \tilde y + \mathcal{C},$ a straightforward geometric estimate computing the distance to the cone
gives
$$d_2(y,A_0^c) \geq d_2(y,\tilde y + \mathcal{C}) = \frac{t}{\sqrt{1+L^2}}, $$
where $d_2$ denotes the Euclidean distance. Using that the Euclidean distance is equivalent to the $\sigma$-distance, we find that
$d(y, A_0^c) \geq c_\sigma t.$
By the triangle inequality, $d(y,x)\leq \sigma(-t\, e_d) + t\leq Ct$, concluding \eqref{eqn:pushAway}.

In the general case, one covers $A_0$ by finitely many cylinders where $\partial A_0$ is represented by a Lipschitz graph with uniformly bounded constant, and chooses $\zeta$ sufficiently small to ensure that the preceding graphical argument goes through in a single cylinder.
\end{proof}

To lift the estimate of Proposition \ref{prop:DboundSimple} on $A_\eps^{-\eps}$ to $A_\eps$, we use that the opening or {closing} of an adversarial minimizer is still a minimizer, see Proposition \ref{prop:openclose}.
We now turn to the main theorem in the low regularity setting.

\begin{proof}[Proof of Theorem \ref{thm:main}\eqref{item:lowReg}]
By Proposition \ref{prop:openclose}, ${\rm op}_\eps(A_\eps)$ is a minimizer of \eqref{eqn:advProb} with $$\rho_0 ({\rm op}_\eps(A_\eps)^\eps) = \rho_0(A_\eps^\eps).$$ Since ${\rm op}_\eps(A_\eps)^\eps\subset (A_\eps)^\eps$ and $\{\rho_0>0\} \supset A_0^c$ outside of a Lebesgue-null set {(by non-degeneracy)}, it follows that ${\rm op}_\eps(A_\eps)^\eps \cap A_0^c =  A_\eps^\eps \cap A_0^c$
outside of a Lebesgue-null set. From this and a simple contradiction argument, we find that 
\begin{equation}\label{eqn:open=normal1} 
A_\eps^\eps \cap A_0^c\subset \overline{{\rm op}_\eps(A_\eps)^\eps}\cap A_0^c
\end{equation} since both dilations are open. 

Applying Proposition \ref{prop:DboundSimple}, we find for $\eps<\eps_0$ that 
$$\sup_{x \in A_\eps^{-\eps} \cap K} d(x,A_0) < C\sqrt{\eps}. $$
 But by \eqref{eqn:open=normal1}, $A_\eps \cap A_0^c \subset \overline{{\rm op}_\eps(A_\eps)^{\eps}} = \overline{(A_\eps^{-\eps})^{2\eps}}$, which leads to
$$ \sup_{x \in A_\eps \cap K} d(x,A_0) \leq (C+2)\sqrt{\eps}$$
for $\eps<\eps_0.$ This gives control on the first term appearing in $d_{\mathcal{H},K}(A_\eps,A_0)$ in \eqref{eqn:HausdorffDef}. 

To control the other term $$\sup_{x \in A_0 \cap K} d(x,A_\eps),$$ as the above argument is symmetric with respect to the labels $i =0$ and $i=1$, we first find that
\begin{equation}\label{eqn:helpyEsty}
\sup_{x \in A_\eps^c \cap K^1} d(x,A_0^c) \leq  \bar C \sqrt{\eps}: =(C+2)\sqrt{\eps}
\end{equation} 
for $C>0$ and $\eps<\eps_0,$ which are possibly chosen larger and smaller, respectively. Consider $x\in A_0\cap K$. If $d(x,A_0^c)>\bar C\sqrt{\eps},$ then \eqref{eqn:helpyEsty} necessarily implies $x \in A_\eps$ so that $d(x,A_\eps) = 0$. In the case that $d(x,A_0^c) \leq \bar C\sqrt{\eps},$ we apply Lemma \ref{lem:pushAway} with $t : = 2\bar C\sqrt{\eps}/c$ to find
${y} \in A_0\cap K^1$ such that
$$d({y},x)\leq C \sqrt{\eps} < 1  \quad \text{ and }\quad d({y},A_0^c) > \bar C\sqrt{\eps}$$
for $\eps<\eps_0$ and some $C>0.$
Reusing our above reasoning, this shows that ${y} \in A_\eps$ and thus $d(x,A_\eps)\leq d({y},x) \leq C \sqrt{\eps}$ as desired.

We conclude that $$d_{\mathcal{H},K}(A_\eps,A_0)\leq C{\sqrt{\eps}} \quad \text{ for }\eps < \eps_0.$$
To obtain the estimate for all $\eps\in (0,1]$, take the constant $C$, appearing in the previous estimate, large enough so that $C\sqrt{\eps_0} \geq {\rm diam}(K)$.
\end{proof}

\begin{example}\label{ex:lowReg}
We present a simple example to show that $O(\sqrt{\eps})$ scaling is likely optimal. Take $\sigma  =|\cdot|_2$ and $A_0  = \{x = (x_1,x_2): x_2>|x_1|\}\subset \R^2$. Consider the local energy $$E_\eps[A] := \int_{B(0,1)} |\chi_A - \chi_{A_0}|\, d\rho_1 + \int_{B(0,1)}|\chi_A|\, d\rho_0 + \eps \, {\rm Per}(A;B(0,1)),$$
where $\rho_1(\cdot) = d(\cdot,A_0^c)$ and $\rho_0(\cdot) = d(\cdot,A_0),$
which is a natural approximation of $\mathcal{R}_{\mu,\eps}(A)$ on $B(0,1)$ given Definition \ref{def:nondegen} (see also \cite{bungert2024gammaconv} for the perimeter term). We consider the energy scaling for competitors found by rounding the tip of $A_0$. Precisely, define the competitors 
$$\mathcal{C}_r : = \big(A_0\cap \{x_2 \geq \sqrt{2}r/2\}\big) \cup B((0,\sqrt{2}r),r)$$
with a rounded cone tip.
Then $$E_\eps[\mathcal{C}_r] = \int_{A_0\setminus \mathcal{C}_r}d(x,A_0) \, dx + \eps(2-2r  +(\pi/2) r)  = r^3\left(\int_{A_0\setminus \mathcal{C}_1}d(x,A_0) \, dx\right) + \eps 2 + \eps((\pi/2)-2)r,$$
where the second equality follows from a change of variables.
Optimizing in $r$ gives that $r = C\sqrt{\eps}$ for some $C>0$,  so that $A_{{\eps}} : = \mathcal{C}_{C\sqrt{\eps}}$ minimizes $E_\eps$. As may be directly checked, $d_{\mathcal{H}}(\mathcal{C}_{r}, A_0) = \tilde C r$, and we see that the optimizer of $E_\eps$ over the class $\{\mathcal{C}_r\}_{r>0}$ is at least $O(\sqrt{\eps})$ away from $A_0$ in the Hausdorff metric.

{Described in heuristic terms, ``clipping'' the corner of the Bayes classifier upward a distance $r$ gives a gain in the perimeter of order $\eps r$, while incurring a bulk cost of order $r^3$: these terms are balanced exactly when $r \sim \sqrt{\eps}$.}
\end{example}

\subsection{Convergence in the high regularity setting}\label{subsec:highReg}

In this subsection, we prove Theorem \ref{thm:main}\eqref{item:highReg}. The proof is similar in structure to Theorem \ref{thm:main}\eqref{item:lowReg} with the essential difference being the construction of energetic competitors near $\partial A_0$. Within the proof of Proposition \ref{prop:DboundSimple} (and in particular in Step 3), we chose the ball $\mathbb{B}$ with center $x_*$ and radius on the order of $D = d(x_*,A_0);$ small $D$ makes it more challenging to construct an energetic improvement (as can be seen in the right-hand side of \eqref{eqn:inteLB}). To obtain the optimal estimate, we use $\sigma$-regularity to choose a ball with new center and a radius independent of the distance to $A_0$ which tangentially touches the parallel surface $\{x:d(x,A_0) = D\}$. Phrasing this intuitively, for the low regularity setting we crudely cut out a ball to construct a competitor, but in the high regularity setting we choose a ball whose removal acts as a small perturbation of the decision boundary $\partial A_\eps$ in the normal direction; this difference is illustrated in Figure \ref{fig:barrier-comparison}.

As before, the essential ingredient is the one-sided bound.
\begin{proposition}\label{prop:Dbound}
    Under the hypotheses of Theorem \ref{thm:main}\eqref{item:highReg}, there is $C = C (\rho_i, \sigma, d, K)>0$ and $\eps_0 = \eps_0 (\rho_i,\sigma, d, K)$ such that $$ \sup_{x \in A_\eps^{-\eps} \cap K} d(x,A_0\cap K) < C\eps \quad \text{ for }\eps<\eps_0.$$
\end{proposition}

The proof of Proposition \ref{prop:Dbound} relies on basic properties for $\sigma$-regular sets, see Definition \ref{def:sigmaReg}, which we record in the following proposition. In particular, the proposition shows that a $\sigma$-regular set $A_0$ has a tubular neighborhood around its boundary described in terms of tangential-normal coordinates. For us, these coordinates provide a natural way to shift points `orthogonally' to $\partial A_0.$

\begin{proposition}[Lemma 33 and 34 in \cite{bellettini2001crystalline}]\label{prop:sigmaReg}
Suppose that $A_0$ is a $\sigma$-regular set with $\sigma$-normal vector $\nu$ as in Definition \ref{def:sigmaReg}. Then the following holds:
\begin{enumerate}
\item \emph{(Normal condition).} $\sigma(\nu(x)) = 1.$

\item \emph{(Transversality).} There is $c>0$ so that $|\nu\cdot \nu_{\rm Eucl}|>c$ for $\mathcal{H}^{d-1}$-a.e. $x\in \partial A_0$.

\item \emph{(Foliation).} For any open and precompact set $U\subset \R^d$, there is $\zeta = \zeta(U) > 0 $ such that the map $\Psi:[\partial A_0\cap U]\times (-\zeta,\zeta)\to \{d(x,\partial A_0)<\zeta\}$ defined by $\Psi(y,s) = y + s\nu(y)$ is a bi-Lipschitz map (onto its image) and $\Psi([\partial A_0\cap U] \times \{s\}) \subset \{{\rm sdist}(x,A_0) = s\}. $
\end{enumerate}
\end{proposition}

We can now complete the proof of the one-sided estimate in Proposition \ref{prop:Dbound}.

\begin{proof}[Proof of Proposition \ref{prop:Dbound}]
Let $\zeta\in (0,1)$ be sufficiently small  so that Definition \ref{def:nondegen} and the conclusion of Proposition \ref{prop:sigmaReg} hold on the dilation $K^1$. Suppose towards a contradiction that 
 \begin{equation}\label{eqn:newContraHyp}
 D : = \sup_{x \in A_\eps^{-\eps} \cap K} d(x,A_0) > C_1 d\eps
 \end{equation}
  for $C_1>0$ chosen below (independent of the dimension). Consider $x_0 \in A_\eps^{-\eps} \cap K$ such that $d(x_0,A_0) = D$ which is possible since $A_\eps^{-\eps}$ is closed.
Following the argument of Steps 1 and 2 of the proof of Proposition \ref{prop:Dbound}, we immediately conclude that
\begin{equation}\nonumber
d(x_0,A_0) \leq  \zeta/2 \quad \text{ for } \quad \eps< \eps_0/d
\end{equation}
for $\eps_0>0$ independent of the dimension.

In the tangential-normal coordinates $\Psi$ associated to $A_0$ via Proposition \ref{prop:sigmaReg}, we write $x_0 = \Psi(y_0,D)$ for $y_0\in \partial A_0$. Define $$x_* : = \Psi(y_0,D+\zeta/8) \quad \text{ and } r  = \zeta/8-3\eps/2.$$
We apply Lemma \ref{lem:basComp} to define $\mathbb{A}_\eps = A_\eps\setminus \mathbb{B}^\eps$ with $\mathbb{B} : = B(x_*,r)$ satisfying
\begin{equation}\nonumber 
 \mathcal{R}_{\mu,\eps}(\mathbb{A}_\eps) + \int_{A_\eps^\eps \cap \mathbb{B}}\Big(\rho_0(x) - \rho_1(M(x)) J_M\Big) \,dx \leq \mathcal{R}_{\mu,\eps}({A}_\eps),
\end{equation}
where $M$ and $J_M$ are defined in \eqref{eqn:MJM}. See Figure \ref{fig:barrier-comparison} for a depiction of the construction.
As before, to contradict minimality of $A_\eps$, we show that the integral is positive. The integral domain $A_\eps^\eps\cap \mathbb{B}$ has positive measure since $\tilde x : = \Psi(y_0,D+7\eps/4)$ satisfies $B(\tilde x,\eps/4)\subset B(x_0,2\eps)\subset A_\eps^\eps$ and $d(x_*,\tilde x) = r - \eps/4$ so that $B(\tilde x,\eps/4)\subset \mathbb{B}$.

To show that the integrand is positive, we follow the same reasoning as in \eqref{eqn:inteLB} to find
\begin{equation}\label{eqn:inteLB2}
\rho_0(x) - \rho_1(M(x))J_M \geq \rho_0(x) - \rho_1(x) -  C_{\rm Lip} \eps - C\frac{\eps d}{r}
\end{equation}
for $x\in \mathbb{B}$ and $2\eps d/r \leq 1/2$ where $C>0$ depends only on $\rho_1$, where we have also used that $\rho_1$ is Lipschitz in $\mathbb{B}$ by non-degeneracy.

Furthermore, for $x \in A_\eps^\eps \cap \mathbb{B}$ we have 
\begin{equation}\label{eqn:disty}
d(x,A_0)\geq d(x_*,A_0) - d(x,x_*) \geq D+ \zeta/8 - (\zeta/8 - 3\eps/2) \geq D.
\end{equation}
For the sake of the reader, we remark that it is precisely for choosing a ball $\B$ with radius independent of $\eps$ that satisfies the previous inequality that we use the tangential-normal coordinates $\Psi$. Uniting \eqref{eqn:disty} with \eqref{eqn:nondegen} (and utilizing the $c>0$ from that definition), \eqref{eqn:newContraHyp}, and \eqref{eqn:inteLB2}, we have
\begin{equation}\label{eqn:helpyBdd}
\rho_0(x) - \rho_1(M(x))J_M \geq cC_1 d\eps - Cd\eps
\end{equation}
for $\eps \leq \eps_0/d$
where $\eps_0>0$ and $C>0$ depend only on $\zeta$ and $\rho_1$ (through its Lipschitz constant and magnitude). For $C_1$ large enough, this ensures the integrand is positive, contradicting minimality of $A_\eps$, and thereby concluding the theorem. 
\end{proof}

At this point it is possible to prove Theorem \ref{thm:main}\eqref{item:highReg} exactly as for Theorem \ref{thm:main}\eqref{item:lowReg}, however the application of Lemma \ref{lem:pushAway} uses equivalence of $\sigma$ and $|\cdot|_2$, which potentially introduces a hidden dependence on the dimension, e.g. when $\sigma = |\cdot|_\infty$. However, the following remark shows that Lemma \ref{lem:pushAway} is superfluous with $\sigma$-regularity.

\begin{remark}\label{rem:pushAway2}
For $A_0$ with locally $\sigma$-regular boundary, let $\zeta$ be as in Proposition \ref{prop:sigmaReg} on $K^1$ with $\Psi$ the associated tangential-normal coordinates. For any $x = \Psi(y,s) \in A_0\cap K$ and $t\in (0,\zeta/2)$ such that $d(x,A_0^c) < t,$ we can define
$x_t : = \Psi (y,s - t) \in A_0$ which satisfies 
\begin{equation}\nonumber
d(x_t,x)= t  \quad \text{ and }\quad d(x_t,A_0^c) \geq t.
\end{equation}
\end{remark} 

Using Remark \ref{rem:pushAway2} in place of Lemma \ref{lem:pushAway}, the proof of Theorem \ref{thm:main}\eqref{item:highReg} follows exactly as in the proof of Theorem \ref{thm:main}\eqref{item:lowReg} with $\sqrt{\eps}$ replaced by $\eps.$

\begin{example}\label{ex:highReg}
We present a simple example to show that $O({\eps})$ scaling is the best expected. Take $\sigma  = |\cdot|_2$ and $A_0 = B(0,1)$ and the class of competitors $\mathcal{C}_r : = B(0,1-r)$. Then with 
$$E_\eps[A] := \int_{B(0,2)} |\chi_A - \chi_{A_0}|\, d\rho_1 + \int_{B(0,2)}|\chi_A|\, d\rho_0 + \eps \,  {\rm Per}(A;B(0,2)),$$
where $\rho_1(\cdot) = d(\cdot,A_0^c)$ and $\rho_0(\cdot) = d(\cdot,A_0).$ Applying polar coordinates
$$E_\eps[\mathcal{C}_r] = \int_{A_0\setminus \mathcal{C}_r}d(x,A_0)\, dx + \eps 2\pi (1-r) \approx \pi (1-r)^2 +  \eps 2\pi (1-r).$$
Thus the optimal $r = 1-C\eps$, and the optimizer of $E_\eps$ over the class $\{\mathcal{C}_r\}_{r>0}$ is at least $O({\eps})$ away from $A_0$ in the Hausdorff metric. See also the closely related Example 6.1 in \cite{trillos2022adversarial}.
\end{example}

\section{The empirical problem}\label{sec:emp}

We turn our attention now to proving Theorem \ref{thm:mainEmpirical} for the empirical adversarial training problem. Similar to before, our argument is based on the construction of a generic competitor for classifiers defined on the discrete domain $X_n = \{x_j : j\leq n\}$ constructed from the data $\{(x_j,y_j)\}_{j\in \N}$ sampled from $\mu$. While the general idea is the same as Lemma \ref{lem:basComp}---excise a ball---it is complicated by the gap between the empirical and continuum distributions. On the path to controlling this error, we first initialize the probabilistic setting. Once we have control on the distance between the empirical and continuum distributions, we may construct the competitor and then address convergence in low and high regularity settings. Throughout, we remind the reader of the conventions introduced in Section \ref{sec:prelim}: dilations/erosions in the empirical adversarial training problem are in the discrete domain $X_n$; however, to connect the empirical problem to the continuum one, we also use dilations/erosions in $\R^d$ (\emph{always} in terms of the norm $\sigma$). For clarity, we typically specify which is used.

\subsection{Probabilistic setting} \label{subsec:prob}

With the conventions introduced in Subsection \ref{subsec:mainEmp}, we have that $\{y_j\}_{j\in \N}$ form a sequence of independent samples drawn from a Bernoulli distribution with probabilities $p  =\rho_0(\Omega)$ and $q = \rho_1(\Omega)$ of obtaining $0$ and $1$ respectively. Letting $n_i$ be the number of samples in $\{(x_j,y_j)\}_{j\leq n}$ with $y_j = i$, we see that $n_1 = \sum_{j\leq n} y_j$. Applying Hoeffding's inequality, we have
$$\mathbb{P}\left(\left|\frac{n_1}{n} - \rho_1(\Omega)\right| > \eta\right)\leq 2\exp(-2n\eta^2).$$
Using the definition of $\eta = \eta(n)$ in \eqref{eqn:distrError}, one can directly check that $\exp(-2n\eta^2) \leq n^{-\alpha/2}$ for a given $\alpha$ with $n$ larger than a universal constant (depending on $\alpha$ but independent of $d$). 
This shows that
\begin{equation}\label{eqn:expectationsClose}
 \left|\frac{n_i}{n} - \rho_i(\Omega)\right| < \eta \quad \text{ with probability } \quad 1-O(n^{-\alpha/2}) 
 \end{equation}
for $i=0,1$,
where we have also used $\left|\frac{n_1}{n} - \rho_1(\Omega)\right| = \left|\frac{n_0}{n} - \rho_0(\Omega)\right|$.
The above inequality also implies that there is $\theta = \theta (\rho_0,\rho_1) \in (0,1/2)$ such that
\begin{equation}\label{eqn:comparenni}
 n\geq n_i\geq \theta n \quad \text{ with probability } \quad 1-O(n^{-\alpha/2})
 \end{equation}
for $i=0,1.$

We derive some transport estimates connecting $\rho_i$ and $\nu_i.$
Applying Theorem 1.1 from \cite{trillos2015rate} and \eqref{eqn:comparenni} with the sampling probability measure $\mathbb{P}$,
it holds that 
$$
\begin{aligned}
\mathbb{P}(d_\infty\left(\tilde \nu_i , \tilde \rho_i\right) \lesssim \eta)  &\geq \mathbb{P}(d_\infty\left(\tilde \nu_i , \tilde \rho_i\right) \lesssim \eta \ \text{ and } \  n_i \geq \theta n)\\
&= \mathbb{P}(d_\infty\left(\tilde \nu_i , \tilde \rho_i\right) \lesssim \eta(n_i) \ | \  n_i \geq \theta n)\, \mathbb{P}(n_i \geq \theta n) \geq  \big(1-O((\theta n)^{-\alpha/2})\big)\big(1-O(n^{-\alpha/2})\big),
\end{aligned}$$
where $d_\infty$ is the $\infty$-Wasserstein metric, the normalized measures are $\tilde \rho_i : = \frac{1}{\rho_i(\Omega)}\rho_i$ and $\tilde \nu_i : = \frac{n}{n_i}\nu_i$, and $\eta$ is defined in \eqref{eqn:distrError} (with the convention $\eta = \eta(n)$).
We summarize this as
\begin{equation}\label{eqn:wassInfty}
d_\infty\left(\tilde \nu_i , \tilde \rho_i\right) \lesssim \eta 
\quad \text{ with probability } \quad 1-O(n^{-\alpha/2}) .
\end{equation}
For convenience, we ignore the constant implicit in \eqref{eqn:wassInfty}, which can be absorbed into $c_0$ and $c_\kappa$ from Theorem \ref{thm:mainEmpirical}. We note that since $\rho_i$ has a continuous density that estimate \eqref{eqn:wassInfty} is equivalently described in terms of a transport map $T_i:\Omega\to \Omega$ satisfying
\begin{equation}\label{eqn:wassInftyTransport}
\|T_i - {\rm id}\|_{\infty}  := \|\sigma(T_i(\cdot) - \cdot)\|_{L^\infty(\Omega)} < \eta \quad \text{ and } \quad \tilde \nu_i = (T_i)_\# \tilde \rho_i, 
\end{equation}
where $\#$ indicates the pushforward. To move from \eqref{eqn:wassInfty}---which is in terms of $|\cdot|_2$---to \eqref{eqn:wassInftyTransport}, we must absorb a further constant into $\eta$ due to the equivalence of $\sigma$ and $|\cdot|_2$. From \eqref{eqn:wassInftyTransport}, we also deduce
\begin{equation}\label{eqn:netnu}
\sup_{x \in \Omega}d(x,{\rm spt} (\nu_i)) < \eta
\end{equation}
for $i=0,1$ since ${\rm spt} (\rho_i )  = \overline\Omega$ by non-degeneracy in $\Omega$.

Finally we use the above estimates to compare $\rho_i(B)$ and $\nu_i(B)$ for an arbitrary set $B\subset \Omega.$
Using \eqref{eqn:expectationsClose} and \eqref{eqn:wassInfty} we estimate that, with probability at least $1-O(n^{-\alpha/2})$,
\begin{equation}\label{eqn:wassCollar}
\begin{aligned}
\left| \nu_i(B) - \rho_i(B) \right| & \leq  {\frac{n_i}{n}}|\tilde\nu_i(B) - \tilde\rho_i(B) | + \left|\frac{n_i}{n}-\rho_i(\Omega) \right|\tilde \rho_i(B)  \\
& \leq {\frac{n_i}{n}}\left|\int_\Omega \left(\chi_B - \chi_B\circ T_i \right) d\tilde \rho_i\right| +\eta \tilde\rho_i(B)   \\
& \leq  {\frac{n_i}{n}}\tilde \rho_i(B^{\eta} \setminus B^{-\eta}) +\eta \tilde\rho_i(B) \\
& \leq  \rho_i(B^{\eta} \setminus B^{-\eta})+ \left({\frac{n_i}{n}} - \rho_i(\Omega)\right)\tilde \rho_i(B^{\eta} \setminus B^{-\eta}) +\eta \tilde\rho_i(B) \\
&\leq \rho_i(B^{\eta} \setminus B^{-\eta}) +C\eta \rho_i(B^\eta).
\end{aligned}
\end{equation}
We emphasize that dilations/erosions are taken in $\R^d.$

In the \emph{event} that \eqref{eqn:expectationsClose} and \eqref{eqn:wassInfty} hold, and thereby \eqref{eqn:expectationsClose}--\eqref{eqn:wassCollar} hold, we now continue in our effort to prove Theorem \ref{thm:mainEmpirical}. 

\subsection{Empirical energetic competitor}
We develop a flexible notion of energetic competitor. While the strategy in the continuum setting is similar, we must now average over radii to manage the error between distributions appearing in \eqref{eqn:wassCollar}, see Figure \ref{fig:avgArg}. 

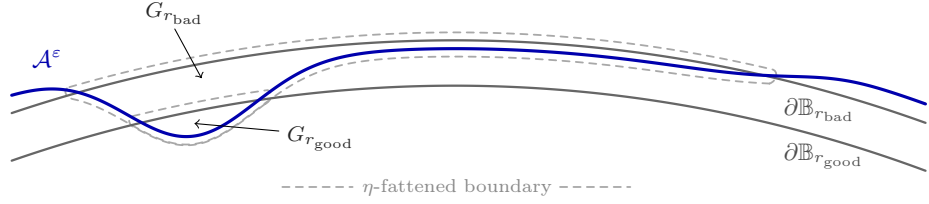
\begin{figure}
%
\begin{tikzpicture}[
  x=.93cm,y=1cm,
  line cap=round,line join=round,
  font=\small,
  collar/.style={black!34,line width=.68pt,dash pattern=on 2.7pt off 2.35pt},
  Abdry/.style={blue!68!black,line width=1.18pt},
  Bbdry/.style={black!60,line width=.86pt}
]

\pgfmathdeclarefunction{ygood}{1}{%
  \pgfmathparse{-18.6 + sqrt(20.5*20.5-(#1-6.6)*(#1-6.6))}%
}
\pgfmathdeclarefunction{ybad}{1}{%
  \pgfmathparse{-18.6 + sqrt(21.1*21.1-(#1-6.6)*(#1-6.6))}%
}

%
%
\pgfmathdeclarefunction{hoff}{1}{%
  \pgfmathparse{%
    0.25
    -0.36*exp(-((#1-6.5)/5.1)^8)
    -0.85*exp(-((#1-2.9)/1.25)^2)%
  }%
}
\pgfmathdeclarefunction{abdry}{1}{%
  \pgfmathparse{ybad(#1)+hoff(#1)}%
}

\def\xgL{2.14882}
\def\xgR{3.81472}
\def\xbL{1.24834}
\def\xbR{10.99577}

\draw[collar]
  (2.02,{ygood(2.02)+.115})
  plot[domain=2.02:3.94,samples=70,smooth] (\x,{ygood(\x)+.115})
  .. controls (4.02,{ygood(3.94)+.09}) and (4.00,{abdry(3.91)-.08}) ..
     (3.90,{abdry(3.90)-.115})
  plot[domain=3.90:2.06,samples=70,smooth] (\x,{abdry(\x)-.115})
  .. controls (1.96,{abdry(2.06)-.08}) and (1.94,{ygood(2.02)+.09}) ..
     (2.02,{ygood(2.02)+.115});

\draw[collar]
  (1.12,{ybad(1.12)+.105})
  plot[domain=1.12:11.12,samples=190,smooth] (\x,{ybad(\x)+.105})
  .. controls (11.21,{ybad(11.10)+.08}) and (11.19,{abdry(11.08)-.07}) ..
     (11.08,{abdry(11.08)-.105})
  plot[domain=11.08:1.16,samples=190,smooth] (\x,{abdry(\x)-.105})
  .. controls (1.05,{abdry(1.16)-.07}) and (1.03,{ybad(1.12)+.08}) ..
     (1.12,{ybad(1.12)+.105});

\draw[Bbdry]
  plot[domain=.30:13.30,samples=220,smooth] (\x,{ygood(\x)});
\draw[Bbdry]
  plot[domain=.30:13.30,samples=220,smooth] (\x,{ybad(\x)});

\draw[Abdry]
  plot[domain=.30:13.30,samples=260,smooth] (\x,{abdry(\x)});

\node[blue!68!black,anchor=south west] at (.46,2.02) {$\mathcal A^{\eps}$};

\node[font=\footnotesize,text=black!86] at (4.7,1.2)
  {$G_{r_{\rm good}}$};
\node[font=\footnotesize,text=black!86,anchor=north] at (2.62,3.12)
  {$G_{r_{\rm bad}}$};
\draw [->] (2.62,2.65) -- (3,2.0);
\draw [->] (4.1,1.3) -- (2.9,1.4);

\node[anchor=west,text=black!70] at (11.15,.97)
  {$\partial\mathbb B_{r_{\rm good}}$};
\node[anchor=west,text=black!70] at (11.15,1.58)
  {$\partial\mathbb B_{r_{\rm bad}}$};

\node[font=\scriptsize,text=black!46,anchor=west] at (5.15,.56)
  {$\eta\text{-fattened boundary}$};
  \draw [collar] (4.15,.56) -- (5.15,.56);
\draw [collar] (8.1,.56) -- (9.1,.56);

\end{tikzpicture}
\caption{
For the good radius, the fattened boundary is controlled by the volume of $G_{r_{\rm good}}  = \mathcal{A}^\eps \cap \mathbb{B}_{r_{\rm good}}$, with $\mathcal{L}^2(G_{r_{\rm good}}) \gtrsim \mathcal{L}^2( G_{r_{\rm good}}^\eta\setminus G_{r_{\rm good}}^{-\eta}).$ However, in the case of the bad radius $r_{\rm bad}$, we see that a thin sliver of $G_{r_{\rm bad}}$ may generate a disproportionately large fattened boundary and break the previous estimate. However, these bad slivers appear only when large pieces of $\partial \mathcal{A}^\eps$ align with $\partial \B_r$, which fails for the majority of radii.}
\label{fig:avgArg}
\end{figure}

\begin{lemma}[Empirical competitor]\label{lem:basCompEmp}
Suppose the results of Subsection \ref{subsec:prob} hold {with the distributional error $\eta$ defined in \eqref{eqn:distrError}}.
Let $\mathcal{A} \subset X_n$, $\eps>0$, $0< \eta< \eps/(8d)$, and $\mathcal{A}^\eps$ denote the dilation in $\R^d$. Given $x_* \in \R^d$ we define the collection of competitors
$$\mathbb{A}_r : = \mathcal{A} \setminus \mathbb{B}_r^\eps \quad \text{ with } \quad \mathbb{B}_r :  = B (x_*,r).$$  For $ R_0\geq 4\eps> 0 $ satisfying
\begin{equation}\label{eqn:volHyp}
0< \mathcal{L}^d(\mathcal{A}^\eps \cap \B_{R_0+7\eps}) \leq Q \mathcal{L}^d( \mathcal{A}^\eps \cap \B_{R_0}),
\end{equation} there is $r\in [R_0,R_0 + 7\eps]$ such that with
$\eps_\eta : = \eps -2\eta$,
$$\lambda_r : = \frac{r +2\eps_\eta}{r}, \qquad M_r (x) : = x_* + \lambda_r(x-x_*), \quad \text{ and }\quad  J_{M_r} : = \lambda_r^{d}$$
it holds that
\begin{equation}\label{eqn:energyImprovementEmp}
\mathcal{R}_{\nu,\eps}(\mathbb{A}_r) + \int_{\mathcal{A}^\eps \cap \mathbb{B}_r}\left[\Big(\rho_0(x) - \rho_1(M_r(x)) J_{M_r}\Big)  - CQ\frac{\eta}{\eps}\right]\,dx  \leq \mathcal{R}_{\nu,\eps}(\mathcal{A}),
\end{equation}
where $C = C(\rho_i,\sigma,d)>0$. 
\end{lemma}

{We emphasize that this is a conditional competitor: given \eqref{eqn:volHyp}, we may construct $\mathbb{A}_r$ satisfying \eqref{eqn:energyImprovementEmp}. Our choice of $Q$ is made in the proof of Theorem \ref{thm:mainEmpirical} and depends on both $\eps$ and the regularity of $A_0$ (low or high).}

To prove the competitor lemma, we rely on a geometric inequality controlling errors arising in \eqref{eqn:wassCollar}. This is a consequence of Theorem 4 in \cite{stacho1976volume}, but given that we only rely on the inequality for discrete sets we provide a self-contained proof in Section \ref{sec:appGeo}.

\begin{proposition}\label{prop:geometricIneq}
Let $S\subset \R^d$ be finite. If $0<\eta\leq \eps/(8d)$, then
$$\mathcal{L}^d((S^\eps)^\eta\setminus (S^\eps)^{-\eta}) \leq Cd\frac{\eta}{\eps}\mathcal{L}^d(S^\eps) $$
for a universal constant $C>0$, where dilations/erosions are taken in $\R^d.$ 
\end{proposition}

We now construct the competitor.

\begin{proof}[Proof of Lemma \ref{lem:basCompEmp}]
We break the argument into two pieces. In the first step, estimate the gap between the risk for $\mathcal{A}$ and $\mathbb{A}_s$ using set inequalities and transport inequalities due to \eqref{eqn:wassInftyTransport}. This leads to geometric errors that are similar to a Minkowski-type perimeter. In Step 2, we perform an averaging argument to control these geometric errors. Throughout, we exclusively use the dilation/erosion in $\R^d.$

\emph{Step 1 (Breakdown of the energy).}
Since $\A_s \subset \mathcal{A}$ we have $\A_s^\eps\subset \mathcal{A}^\eps$ and $(\mathcal{A}^c)^\eps\subset (\A_s^c)^\eps$, where complements are taken relative to $X_n$, we have that
\begin{equation}\label{eqn:riskRewriteEmp}
 \mathcal{R}_{\nu,\eps}(\mathcal{A}) -\mathcal{R}_{\nu,\eps}(\A_s) = \nu_0(\mathcal{A}^\eps \setminus \mathbb{A}_s^\eps) - \nu_1((\mathbb{A}_s^c)^\eps \setminus (\mathcal{A}^c)^\eps).
 \end{equation}
 To apply \eqref{eqn:wassCollar} to the right-hand side sets, we first modify the sets so that their $\eta$-fattened boundaries are easier to control. Precisely,
 because $\mathbb{A}_s \subset (\B^\eps_s)^c$,  we have that 
 \begin{equation}\label{def:Gs} 
 G_s : = \mathcal{A}^\eps \cap \B_s \subset \mathcal{A}^\eps \setminus \mathbb{A}_s^\eps
 \end{equation}
 {(recall Figure \ref{fig:avgArg}).}
 We further claim that
 \begin{equation}\label{eqn:empSetIncl}
  X_n \cap (\mathbb{A}_s^c)^\eps \setminus (\mathcal{A}^c)^\eps \subset M_s(G_s)\cup \big(\mathcal{A}\cap \B^{2\eps}_{s}\setminus\B^{2\eps_\eta}_{s}\big). 
  \end{equation}

 By definition of $\A_s$, we have that
 $$X_n \cap (\mathbb{A}_s^c)^\eps \setminus (\mathcal{A}^c)^\eps \subset \B^{2\eps}_s.$$ 
 From this, as $X_n \cap (\mathbb{A}_s^c)^\eps \setminus (\mathcal{A}^c)^\eps \subset \mathcal{A}$, it follows that 
 \begin{equation}\label{eqn:outsideBallSubset}
 \left(X_n \cap (\mathbb{A}_s^c)^\eps \setminus (\mathcal{A}^c)^\eps\right)\setminus \B^{2\eps_\eta}_s \subset \mathcal{A}\cap \B^{2\eps}_{s}\setminus\B^{2\eps_\eta}_{s}.
 \end{equation}
 On the other hand, if $$x\in \left(X_n \cap (\mathbb{A}_s^c)^\eps \setminus (\mathcal{A}^c)^\eps\right)\cap \B^{2\eps_\eta}_{s}$$ then $N_s : = M_s^{-1}$ maps $\B^{2\eps_\eta}_s$ onto $\B_s$ so that 
 \begin{equation}\label{eqn:baller1}
 N_s(x)\in \B_s.
 \end{equation} Additionally, since $x\not \in (\mathcal{A}^c)^\eps$, we claim that 
 \begin{equation}\label{eqn:smallBallContain}
 B(x,2\eps_\eta)\subset \mathcal{A}^\eps.
 \end{equation} To see this, let $y\in B(x,2\eps_\eta)$ and define $z = (x+y)/2$. Then by \eqref{eqn:netnu}, there is $\tilde x \in X_n$ such that $d(\tilde x , z)<\eta$. By the triangle inequality $d(x,\tilde x)<\eps$ and $d(\tilde x, y)<\eps$. The first inequality $d(x,\tilde x)<\eps$ implies that $\tilde x\in \mathcal{A}$ (see Proposition \ref{prop:openCloseBasicProperties}), and the second that $y\in \mathcal{A}^\eps$, proving \eqref{eqn:smallBallContain}. As $d(x,N_s(x)) < 2\eps_\eta$ {(see above \eqref{eqn:setInclus})}, we conclude 
$ N_s(x)\in \mathcal{A}^\eps$. This set inclusion combined with \eqref{eqn:baller1} and \eqref{def:Gs} implies that $x \in M_s(G_s)$. With \eqref{eqn:outsideBallSubset}, this conclude \eqref{eqn:empSetIncl}.

Using \eqref{def:Gs} and \eqref{eqn:empSetIncl} in \eqref{eqn:riskRewriteEmp}, we apply \eqref{eqn:wassCollar} to find  
\begin{equation}\label{eqn:almostEnergyBound}
\begin{aligned}
\mathcal{R}_{\nu,\eps}(\mathcal{A}) -\mathcal{R}_{\nu,\eps}(\A_s) & \geq \nu_0(G_s) - \nu_1(M_s(G_s)) - \nu_1(\mathcal{A}\cap \B^{2\eps}_{s}\setminus\B^{2\eps_\eta}_{s} ) \\
& \geq \int_{G_s}\Big(\rho_0(x) - \rho_1(M_r(x)) J_{M_r}\Big)\, dx  -   C \eta\mathcal{L}^d(G_s^\eta) \\
& \qquad  - \Big(\rho_0(G_s^{\eta}\setminus G_s^{-\eta}) + \rho_1([M_s(G_s)]^{\eta}\setminus [M_s(G_s)]^{-\eta}) + \rho_1 (\mathcal{A}^\eta\cap \B^{2\eps+\eta}_{s}\setminus\B^{2\eps-5\eta}_{s}) \Big),
\end{aligned}
\end{equation} 
where the constant in the first error term absorbs $J_{M_s}$ and $\|\rho_i\|_\infty$.
We next conduct an averaging argument to choose a good radius $r\in [R_0,R_0+\eps]$ so that
\begin{equation}\label{eqn:simpleBit}
    \mathcal{L}^d(G_r^\eta) \leq Q\mathcal{L}^d(G_r)
\end{equation}
and
\begin{equation} \label{eqn:estViaAvg}
\rho_0(G_r^{\eta}\setminus G_r^{-\eta}) + \rho_1([M_r(G_r)]^{\eta}\setminus [M_r(G_r)]^{-\eta}) + \rho_1 (\mathcal{A}^\eps\cap \B^{2\eps+\eta}_{r}\setminus\B^{2\eps-5\eta}_{r})\leq Q\frac{\eta}{\eps}\mathcal{L}^d(G_r),
\end{equation} which with \eqref{def:Gs} and \eqref{eqn:almostEnergyBound} concludes the lemma.

\emph{Step 2 (Choosing a good radius $r$).}
Define the set $S : = \mathcal{A}\cap \B_{R_0+6\eps}$. For $x\in (\mathcal{A}^\eps)^\eta \cap \B_{R_0+4\eps},$ there is some $y\in \mathcal{A}$ with $d(x,y) < \eps + \eta < 2\eps$, meaning $y\in  \B_{R_0+6\eps}$. Consequently $x\in (S^{\eps})^\eta$, and $(\mathcal{A}^\eps)^\eta \cap \B_{R_0+4\eps}\subset (S^{\eps})^\eta$. Simpler reasoning gives $(S^{\eps})^{-\eta}\subset (\mathcal{A}^\eps)^{-\eta}$. These subset relations in conjunction with Proposition \ref{prop:geometricIneq} give
\begin{equation}\label{eqn:firstPerimEst}
\mathcal{L}^d\big([(\mathcal{A}^\eps)^\eta \setminus (\mathcal{A}^\eps)^{-\eta}] \cap \B_{R_0+4\eps} \big)\leq C \frac{\eta}{\eps}\mathcal{L}^d(S^\eps) \leq C \frac{\eta}{\eps}\mathcal{L}^d(\mathcal{A}^\eps \cap \B_{R_0+7\eps}) = C \frac{\eta}{\eps}\mathcal{L}^d(G_{R_0+7\eps}). 
\end{equation}
As a direct consequence of this inequality and $\eta\leq\eps$, we also find 
\begin{equation}\label{eqn:anotherVol}
\mathcal{L}^d\big((\mathcal{A}^\eps)^\eta  \cap \B_{R_0+4\eps} \big) 
\leq \mathcal{L}^d\big((\mathcal{A}^\eps)^\eta\setminus (\mathcal{A}^\eps)^{-\eta}  \cap \B_{R_0+4\eps} \big) +\mathcal{L}^d\big(\mathcal{A}^\eps  \cap \B_{R_0+4\eps} \big)
\leq  
C \mathcal{L}^d(G_{R_0+7\eps}).
\end{equation}
We note this inequality with \eqref{eqn:volHyp} shows that for any $r \in [R_0,R_0+\eps]$,  \eqref{eqn:simpleBit} holds.

Recalling \eqref{def:Gs}, we have that
\begin{equation} \label{eqn:graphDecomp}
G_s^{\eta}\setminus G_s^{-\eta} \subset \Big( \left[(\mathcal{A}^\eps)^\eta \setminus (\mathcal{A}^\eps)^{-\eta}\right]\cap \B_s^\eta \Big) \cup \Big((\mathcal{A}^\eps)^\eta \cap \B_s^\eta \setminus \B_s^{-\eta}\Big).
\end{equation}
Applying \eqref{eqn:firstPerimEst} and \eqref{eqn:volHyp} for any $s\in [R_0,R_0+\eps]$, it holds that 
$$\mathcal{L}^d\Big( \left[(\mathcal{A}^\eps)^\eta \setminus (\mathcal{A}^\eps)^{-\eta}\right]\cap \B_s^\eta \Big) \leq  C Q\frac{\eta}{\eps}\mathcal{L}^d(G_{R_0}) \leq C Q\frac{\eta}{\eps}\mathcal{L}^d(G_s).$$
To control the last set in \eqref{eqn:graphDecomp}, we use polar coordinates/Fubini's theorem, \eqref{eqn:anotherVol}, and \eqref{eqn:volHyp} to estimate
\begin{equation}\label{eqn:avgArgument}
\int_{R_0}^{R_0+\eps} \mathcal{L}^d\Big((\mathcal{A}^\eps)^\eta \cap \B_s^\eta \setminus \B_s^{-\eta}\Big) \, ds \leq C\eta \mathcal{L}^d\big((\mathcal{A}^\eps)^\eta  \cap \B_{R_0+4\eps} \big) \leq C Q\eta\mathcal{L}^d(G_{R_0}).
\end{equation}
An averaging argument then gives that there is $r \in [R_0,R_0+\eps]$
such that
$$\mathcal{L}^d\Big((\mathcal{A}^\eps)^\eta \cap \B_r^\eta \setminus \B_r^{-\eta}\Big) \leq C Q\frac{\eta}{\eps}\mathcal{L}^d(G_{R_0})\leq C Q\frac{\eta}{\eps}\mathcal{L}^d(G_{r}).$$
Altogether we find
\begin{equation}\label{eqn:anVol2}
\mathcal{L}^d(G_r^{\eta}\setminus G_r^{-\eta})\leq C Q\frac{\eta}{\eps}\mathcal{L}^d(G_{r}),
\end{equation}
giving control on the first term in \eqref{eqn:estViaAvg}.

Control on the second term in \eqref{eqn:estViaAvg} follows from \eqref{eqn:anVol2}: Since $M_r$ is isotropic, we have that
$$[M_r(G_r)]^{\eta}\setminus [M_r(G_r)]^{-\eta} = M_r(G_r^{\eta/\lambda_r}\setminus G_r^{-\eta/\lambda_r})$$
and so $$\mathcal{L}^d ([M_r(G_r)]^{\eta}\setminus [M_r(G_r)]^{-\eta}) = J_{M_r}\mathcal{L}^d(G_r^{\eta/\lambda_r}\setminus G_r^{-\eta/\lambda_r}) \leq C Q\frac{\eta}{\eps}\mathcal{L}^d(G_{r}), $$
where in the last inequality we use \eqref{eqn:anVol2} and $1\leq \lambda_r \leq 2$ to bound $J_M\leq 2^d$ (recall $R_0\geq 4\eps$; we point out that a much tighter control of $J_{M_r}$ is possible within the proof of Theorem \ref{thm:mainEmpirical}).

Finally the third term in \eqref{eqn:estViaAvg} may be controlled using another averaging argument as in \eqref{eqn:avgArgument}. As the averaging arguments may be done simultaneously, we find a good radius $r\in [R_0,R_0+\eps]$ obtaining \eqref{eqn:estViaAvg}, concluding the proof.
\end{proof}

\subsection{Convergence in the low regularity setting}\label{subsec:lowRegEmp}
We turn to the proof of Theorem \ref{thm:mainEmpirical}\eqref{item:lowRegEmp}. As the proof follows the same strategy as that for Theorem \ref{thm:main}\eqref{item:lowReg}, we are light on details where they are in close correspondence and provide full details where the proofs deviate.

We first obtain a one-sided estimate. The primary difference with respect to the continuum setting is an averaging argument to choose $Q$ in Lemma \ref{lem:basCompEmp} well. For the low regularity setting, we are able to choose $Q =2.$ Looking to \eqref{eqn:energyImprovementEmp} with $\eta  = O(\eps^{3/2})$ and non-degeneracy in mind shows that a contradiction to minimality arises if we are at least $C\sqrt{\eps}$ away from $\partial A_0$ for $C\gg 1$. We formalize this argument in the following proposition.

\begin{proposition}\label{prop:DboundSimEmp}
Under the hypotheses of Theorem \ref{thm:mainEmpirical}\eqref{item:lowRegEmp}, there is $C = C (\rho_i, \sigma, d, K)>0$ and $\eps_0 = \eps_0 (\rho_i,\sigma, d, K)$ such that $$ \sup_{x \in \mathcal{A}_\eps^{-\eps} \cap K} d(x,A_0) < C\sqrt{\eps} \quad \text{ for }\eps<\eps_0,$$
where $\mathcal{A}_\eps^{-\eps}$ is the erosion in $X_n$.
\end{proposition}

\begin{proof}
Dilations are taken in $\R^d$.
Let $K\subset\subset K_1 \subset\subset \Omega$, with $K_1$ an open set. Choose $\zeta\in (0,{\rm dist}(K,\partial K_1))$  as in the non-degeneracy Definition \ref{def:nondegen} on $K_1$.
Suppose by way of contradiction that 
\begin{equation}\label{eqn:contraEmp1}
D : = \sup_{x \in \mathcal{A}_\eps^{-\eps} \cap K} d(x,A_0) > C_0 \sqrt{\eps}
\end{equation}
for $C_0$ to be chosen below. Let $x_*\in \mathcal{A}_\eps^{-\eps}$ satisfy $d(x_*,A_0) = D$ and define $r_* : = \min\{D/8,\zeta/16\}$. 
For $\eps < \eps_0 (\zeta)$, we have $r_* > 4\eps,$ so
in particular, we have that 
\begin{equation}\label{eqn:anotherVolLowBdd}
\mathcal{L}^d(A_\eps^\eps \cap \B_{r_*}) \geq \mathcal{L}^{d}(\B_\eps) >c \eps^d
\end{equation} where $\B_{s} : = B(x_*,s)$.
Defining $r_j : = r_*+ 7\eps j$, we claim that for $\eps<\eps_0$ (determined by \eqref{eqn:polyBlowup}) there is some $r_{j_0} \in [r_*,3r_*/2]$ such that
\begin{equation}\label{eqn:volHypSetup1}
\mathcal{L}^d(A_\eps^\eps \cap \B_{r_{j_0+1}})\leq 2 \mathcal{L}^d(A_\eps^\eps \cap \B_{r_{j_0}}).
\end{equation}
If the inequality fails for every $r_j$, we can iterate the negated inequality across $j=0,\ldots, \lfloor {r_*/(14\eps)} \rfloor$ to find
$$2^{\lfloor {r_*/(14\eps)} \rfloor}\mathcal{L}^d(A_\eps^\eps \cap \B_{r_{*}}) \leq \mathcal{L}^d(A_\eps^\eps \cap \B_{(3/2)r_{*}}).$$
Combined with \eqref{eqn:anotherVolLowBdd} we obtain
$$2^{r_*/28\eps}\eps^d \leq Cr_*^d, $$
for $C>0$ (depending also on the dimension). Inserting \eqref{eqn:contraEmp1} and the definition of $r_*$, we find that
\begin{equation}\label{eqn:polyBlowup}
2^{C\eps^{-1/2}}\leq C(r_*/\eps)^d\leq C \eps^{-d},
\end{equation} which is not possible for $\eps$ sufficiently small (independent of $C_0$ when $C_0 \geq 1$).
It follows that \eqref{eqn:volHypSetup1} holds for some choice of  $j_0$.

Consequently using $\eta < c_0\eps^{3/2}$, we may apply Lemma \ref{lem:basCompEmp} with $R_0 : = r_{j_0}$ and $Q = 2$ to find a competitor denoted (with an abuse of notation) by $\mathbb{A}_\eps : = \mathcal{A}_\eps\setminus \B_{r}$ for some $r\in [r_*,2r_*]$ such that
 \begin{equation}\label{eqn:energyImprovementEmp1}
\mathcal{R}_{\nu,\eps}(\mathbb{A}_\eps) + \int_{\mathcal{A}_\eps^\eps \cap \mathbb{B}_r}\left[\Big(\rho_0(x) - \rho_1(M(x)) J_{M}\Big)  - c_0C\sqrt{\eps}\right]\,dx  \leq \mathcal{R}_{\nu,\eps}(\mathcal{A}_\eps),
\end{equation}
with $\eps_\eta : = \eps- 2\eta$,
$$\lambda : = \frac{r +2\eps_\eta}{r}, \qquad M (x) : = x_* + \lambda(x-x_*), \quad \text{ and }\quad  J_{M} : = \lambda^{d}.$$

Since by construction $r \leq \min\{D/4,\zeta/4\}$, one may proceed precisely as in the proof of Proposition \ref{prop:DboundSimple}:
In the \emph{far from} case, 
$d(x_*,A_0)\geq \zeta/2$ and it holds that
$$\rho_0(x) - \rho_1(M(x)) J_{M} \geq \gamma/4 \quad \text{ where }\quad \gamma : = \inf_{x \in K_1 \cap \{d(y,A_0) \geq  \zeta/4\}}(\rho_0-\rho_1) > 0 $$ for $\eps$ sufficiently small (this follows from \eqref{eqn:inteLB} and \eqref{eqn:helpfuLB}). When $C\sqrt{\eps} < \gamma/4$, 
\eqref{eqn:energyImprovementEmp1} contradicts the minimality of $\mathcal{A}_\eps$ and shows that $d(x_*,A_0)<\zeta/2$.
In the \emph{close to} case with $d(x_*,A_0)<\zeta/2$, given $\tilde C_0 > 0$, for $C_0$ sufficiently large and $\eps$ sufficiently small it holds that
$$\rho_0(x) - \rho_1(M(x)) J_{M} \geq \tilde C_0\sqrt{\eps}$$ (this follows from \eqref{eqn:helpfuLB1}). Thus, taking $c_0$ appearing in \eqref{eqn:energyImprovementEmp1} sufficiently small shows that the integral term therein is positive, contradicting minimality of $\mathcal{A}_\eps$ and concluding the proof. 
\end{proof}

Finally we complete the proof of convergence in the low regularity setting for empirical classifiers. The main difference with the continuum argument is that we must bound the distance to the continuum Bayes classifier, but Proposition \ref{prop:openclose} naturally relates sets to the \emph{empirical} Bayes classifier. We use \eqref{eqn:netnu} to move between the two Bayes classifiers. 

\begin{proof}[Proof of Theorem \ref{thm:mainEmpirical}\eqref{item:lowRegEmp}]
Let $K\subset \subset K_1 \subset \subset K_2 \subset\subset \Omega$ and consider $\eps < \min\{d(K,K_1^c),d(K_1,K_2^c)\}/8.$ Applying Proposition \ref{prop:DboundSimEmp}, we have that
\begin{equation}\label{eqn:letsbounds}
\sup_{x \in \mathcal{A}_\eps^{-\eps} \cap K_2} d(x,A_0) < C\sqrt{\eps} \quad \text{ for }\eps<\eps_0.
\end{equation}
Proposition \ref{prop:openclose} shows that the discrete opening ${\rm op}_\eps (\mathcal{A_\eps})$ is an empirical adversarial minimizer with
$$ {\rm op}_\eps (\mathcal{A_\eps})^\eps \subset \mathcal{A}^\eps_\eps \quad \text{ and } \quad \nu_0({\rm op}_\eps (\mathcal{A_\eps})^\eps) = \nu_0(\mathcal{A}^\eps_\eps).$$
This implies 
\begin{equation}\label{eqn:openSetEmp}
{\rm op}_\eps (\mathcal{A_\eps})^\eps \cap {\rm spt}(\nu_0) = \mathcal{A}^\eps_\eps \cap {\rm spt}(\nu_0) .
\end{equation}
Since ${\rm op}_\eps (\mathcal{A_\eps})^{-\eps} = \mathcal{A}^{-\eps}_\eps$ by Proposition \ref{prop:openCloseBasicProperties}, \eqref{eqn:letsbounds} and \eqref{eqn:openSetEmp} give
\begin{equation}\label{eqn:letsbounds2}
\sup_{x \in \mathcal{A}_\eps^\eps \cap K_1 \cap {\rm spt}(\nu_0)} d(x,A_0) \leq \sup_{x \in {\rm op}_\eps (\mathcal{A_\eps})^{-\eps} \cap K_2 \cap {\rm spt}(\nu_0)} d(x,A_0) +2\eps < C\sqrt{\eps} +2\eps \leq C\sqrt{\eps} \quad \text{ for }\eps<\eps_0,
\end{equation}
where we have possibly increased $C.$

On the other hand, consider $x \in \mathcal{A}_\eps \cap K\cap {\rm spt}(\nu_1)$. By \eqref{eqn:netnu} there is $y\in {\rm spt}(\nu_0)$ with $d(x,y)<\eta$ so that, in particular, $y\in \mathcal{A}^\eps_\eps \cap K_1 \cap {\rm spt}(\nu_0) $. Then \eqref{eqn:letsbounds2} implies 
\begin{equation}\label{eqn:letsbounds3}
\sup_{x \in \mathcal{A}_\eps \cap K } d(x,A_0) <  C\sqrt{\eps} \quad \text{ for }\eps<\eps_0,
\end{equation}
where we have possibly increased $C.$
This gives control of the first term of $d_{\mathcal{H},K}(\mathcal{A}_\eps,A_0)$ in \eqref{eqn:HausdorffDef}.

We obtain control on the second term $$\sup_{x \in {A_0} \cap K } d(x, {\mathcal{A}_\eps})$$ within the Hausdorff distance: Interchanging labels, the same argument as for \eqref{eqn:letsbounds3} gives 
\begin{equation}\label{eqn:letsbounds4}
\sup_{x \in \mathcal{A}_\eps^c \cap K_1 } d(x,A_0^c) <  \bar C\sqrt{\eps} \quad \text{ for }\eps<\eps_0
\end{equation}
for some constant $\bar C>0$.
Take $x \in A_0\cap K$. If $d(x,A_0^c)\geq (\bar C + 1)\sqrt{\eps}$, then letting $\tilde x$ be the nearest point in $X_n$ to $x$, we have by \eqref{eqn:netnu} and the triangle inequality that 
$$ d(\tilde x,A_0^c)\geq (\bar C + 1)\sqrt{\eps} - \eta \geq  \bar C \sqrt{\eps}.$$ Since $\eta\ll 1$ so that $\tilde x\in K_1$, from \eqref{eqn:letsbounds4} it necessarily follows that $\tilde x\in \mathcal{A}_\eps$ and 
\begin{equation}\label{eqn:firstcase}
d(x,\mathcal{A}_\eps)<\eta.
\end{equation} If $x$ is close to $\partial A_0$ with $d(x,A_0^c)\leq (\bar C + 1)\sqrt{\eps},$ we may assume that $\eps$ is sufficiently small so that $(\bar C + 1)\sqrt{\eps} < \zeta$, where $\zeta$ is as in Lemma \ref{lem:pushAway}. Consequently, we may find $y$ such that
$$d(y,A_0^c) \geq  (\bar C + 1)\sqrt{\eps} \quad \text{ and } \quad d(x,y)<C\sqrt{\eps}.$$
By \eqref{eqn:firstcase}, we find $d(y,\mathcal{A}_\eps)<\eta$ so that 
$$d(x,\mathcal{A}_\eps) \leq d(x,y) + \eta \leq  C\sqrt{\eps} + \eta \leq (C+1)\sqrt{\eps}.$$
Altogether, we have
$$\sup_{x\in A_0\cap K}d(x,\mathcal{A}_\eps) \leq C\sqrt{\eps}, $$
for a larger constant $C>0$. With \eqref{eqn:letsbounds2}, this concludes the proof.
\end{proof}

\subsection{Convergence in the high regularity setting}
To deduce Theorem \ref{thm:mainEmpirical}\eqref{item:highRegEmp}, it remains to synthesize the strategies developed in Subsections \ref{subsec:highReg} and \ref{subsec:lowRegEmp}. The primary novelty of this section is the averaging argument to select $Q$ within Lemma \ref{lem:basCompEmp}. This selection occurs within the proof of the following one-sided estimate.

\begin{proposition}\label{prop:DboundEmp}
Under the hypotheses of Theorem \ref{thm:mainEmpirical}\eqref{item:highRegEmp}, there is $C = C (\rho_i, \sigma, d, K)>0$ and $\eps_0 = \eps_0 (\rho_i,\sigma, d, K)$ such that $$ \sup_{x \in \mathcal{A}_\eps^{-\eps} \cap K} d(x,A_0) < C{\eps} \quad \text{ for }\eps<\eps_0,$$
{where $\mathcal{A}_\eps^{-\eps}$ is the erosion in $X_n$.}
\end{proposition}

\begin{proof}

Applying Proposition \ref{prop:DboundSimEmp}, we may assume 
$$ D: = \sup_{x \in \mathcal{A}_\eps^{-\eps} \cap K} d(x,A_0) < \zeta/4,$$
where $\zeta\in (0,1)$ is chosen to satisfy Definitions \ref{def:nondegen} and \ref{def:sigmaReg}. Let $x_0 \in \mathcal{A}_\eps^{-\eps} \cap K$ satisfy $d(x_0,A_0) = D$. Using the tangential-normal coordinates $\Psi$ associated to $\partial A_0$ via Proposition \ref{prop:sigmaReg}, we write $x_0 : = \Psi(y_0,D).$ Define 
$$x_* : = \Psi(y_0,D+\zeta/4) \quad \text{ and }\quad R : = \zeta/4-3\eps/2.$$
Defining $\B_s : = B(x_*,s),$ first note that
\begin{equation}\label{eqn:vbme}
\mathcal{L}^d(\mathcal{A}_\eps^\eps \cap \B_{R}) > c\eps^{d}
\end{equation}
for some constant $c = c(d,\sigma)>0$. With $\tilde x : = \Psi(y_0,D+7\eps/4)$, we have $d(\tilde x, x_0) = 7\eps/4$ and $d(\tilde x,x_*) =  R-\eps/4$. Thus $B(\tilde x,\eps/4)\subset \mathcal{A}_\eps^\eps \cap \B_{R}$ proving \eqref{eqn:vbme}.

Define $m : = \lceil d/\kappa \rceil$ and take $\eps$ sufficiently small so that $m7\eps < \zeta/16$, where $\kappa>0$ comes from the statement of Theorem \ref{thm:mainEmpirical}\eqref{item:highRegEmp}. Letting $r_j : = R+ j7\eps$, we claim that there is $j_0\in \{0,\ldots, m-1\}$ such that
\begin{equation}\label{eqn:volDrop}
\mathcal{L}^d(\mathcal{A}_\eps^\eps \cap \B_{r_{j_0+1}})\leq C_{\kappa} \eps^{-\kappa}\mathcal{L}^d(\mathcal{A}_\eps^\eps \cap \B_{r_{j_0}})
\end{equation}
for a constant $C_\kappa = C(\kappa,\zeta,d)>0.$
If not, we iterate the negated estimate to obtain
$$(C_\kappa\eps^{-\kappa})^m  \mathcal{L}^d(\mathcal{A}_\eps^\eps \cap \B_{R}) \leq \mathcal{L}^d(\mathcal{A}_\eps^\eps \cap \B_{3R/2})\leq \mathcal{L}^d( \B_{1}).$$
Using \eqref{eqn:vbme} gives
$$ C_\kappa^m c \leq C_\kappa^m c \eps^{d-\kappa m} \leq \mathcal{L}^d( \B_{1}),$$
leading to a contradiction for sufficiently large $C_\kappa$ and concluding \eqref{eqn:volDrop}.

With \eqref{eqn:volDrop}, we set $R_0 : =r_{j_0}$ and apply Lemma \ref{lem:basCompEmp} with $Q  = C_\kappa \eps^{-\kappa}$ and $\eta <c_\kappa\eps^{2+\kappa}$ to instantiate a competitor $\A_\eps : =\mathcal{A}_\eps\setminus \B_{r}$ for some $r\in [R_0,R_0+7\eps]$ satisfying
\begin{equation}\label{eqn:energyImprovementEmp2}
\mathcal{R}_{\nu,\eps}(\mathbb{A}_\eps) + \int_{\mathcal{A}_\eps^\eps \cap \mathbb{B}_r}\left[\Big(\rho_0(x) - \rho_1(M(x)) J_{M}\Big)  -C c_\kappa\eps\right]\,dx  \leq \mathcal{R}_{\nu,\eps}(\mathcal{A}_\eps),
\end{equation}
with $\eps_\eta : = \eps- 2\eta$,
$$\lambda : = \frac{r +2\eps_\eta}{r}, \qquad M (x) : = x_* + \lambda(x-x_*), \quad \text{ and }\quad  J_{M} : = \lambda^{d}.$$
Since $r \geq \zeta/8$, which is independent of $\eps$, we may follow the argument of Proposition \ref{prop:Dbound} to estimate
$\rho_0(x) - \rho_1(M(x)) J_{M} \geq \bar C\eps$
for $\eps$ sufficiently small (this follows from \eqref{eqn:helpyBdd}). Taking $c_\kappa$ sufficiently small, we see that \eqref{eqn:energyImprovementEmp2} contradicts the minimality of $\mathcal{A}_\eps$ concluding the proof.
\end{proof}

With the one-sided estimate in hand, the proof of Theorem \ref{thm:mainEmpirical}\eqref{item:highRegEmp} follows exactly as for Theorem \ref{thm:mainEmpirical}\eqref{item:lowRegEmp} with Remark \ref{rem:pushAway2} and Proposition \ref{prop:DboundEmp} in place of Lemma \ref{lem:pushAway} and Proposition \ref{prop:DboundSimEmp} (and $\eps$ in place of $\sqrt{\eps}$), respectively.

\appendix

\section{Geometric inequality} \label{sec:appGeo}

\begin{proof}[Proof of Proposition \ref{prop:geometricIneq}]
Let $S : = \{p_i\}_{i=1}^n.$ We begin by proving the auxiliary estimate
\begin{equation}\label{eqn:auxEst}
\mathcal{L}^d(S^u)\geq \left(\frac{u}{v}\right)^d \mathcal{L}^d(S^v) \quad \text{ for } \quad 0< u \leq v.
\end{equation}
We define the Voronoi partition $\{V_i\}_{i=1}^n$ of $\R^d$ with cells 
$$V_i : = \Big\{x : d(x,p_i)<d(x,p_j) \text{ if } j<i \quad \text{ and } \quad  d(x,p_i)\leq d(x, p_j) \text{ if } i<j\Big\}. $$
We point out that $V_i$ is star-shaped with center $p_i$: Given $x\in V_i$ we define
 $y = p_i + t(x- p_i)$ for $t\in (0,1)$ and check for $j<i$ that
 $$d(y,p_j) \geq d(x,p_j) - d(y,x) = d(x,p_j) - (1-t)d(x,p_j) > t d(x,p_i) = d(y,p_i), $$
 with the case $i<j$ being similar.  (Note that a Voronoi cell is not necessarily convex for every norm $\sigma$.)
 
 Note that $S^u = \cup_i\left(V_i \cap B(p_i,u)\right)$ for all $u>0$.
 We define the contraction $N_i(x) : = p_i + \frac{u}{v}(x-p_i),$ and deduce via star-shapedness that 
 $$N_i(V_i \cap B(p_i,v))\subset V_i \cap B(p_i,u).$$ It follows that
 \begin{equation}\nonumber
 \begin{aligned}
\mathcal{L}^d(S^u) &= \sum_i \mathcal{L}^d (V_i \cap B(p_i,u)) \\
& \geq \sum_i \mathcal{L}^d(N_i(V_i \cap B(p_i,v))) = \left(\frac{u}{v}\right)^d \sum_i \mathcal{L}^d(V_i \cap B(p_i,v)) = \left(\frac{u}{v}\right)^d\mathcal{L}^d(S^v),
\end{aligned}
 \end{equation}
 concluding \eqref{eqn:auxEst}.
 
 We let $u = \eps-\eta$ and $v = \eps+\eta$ within \eqref{eqn:auxEst} to obtain that
 $$\mathcal{L}^d(S^{\eps+\eta}\setminus S^{\eps-\eta}) = \mathcal{L}^d(S^{\eps+\eta}) - \mathcal{L}^d(S^{\eps-\eta}) \leq \left[\left(\frac{\eps+\eta}{\eps-\eta}\right)^d - 1 \right]\mathcal{L}^d(S^{\eps-\eta}).$$
 Letting $t= \frac{\eps+\eta}{\eps-\eta} - 1$, using $1+t \leq e^t$ we estimate
 $$ \left(\frac{\eps+\eta}{\eps-\eta}\right)^d - 1 \leq e^{dt} - 1\leq Cdt,$$
 where the last inequality follows if $dt \leq 1/2.$
 To satisfy the condition $dt \leq 1/2$, note $t = \frac{2\eta}{\eps-\eta}$ and we may assume $\eta\leq \eps/2$ so that
 $$dt \leq 4d\frac{\eta}{\eps}.$$ Thus $dt\leq 1/2$ follows from $\eta\leq \eps/(8d)$.
 Uniting the above displays gives
 \begin{equation}\nonumber
 \mathcal{L}^d(S^{\eps+\eta}\setminus S^{\eps-\eta}) \leq Cd\frac{\eta}{\eps}\mathcal{L}^d(S^{\eps-\eta}) \leq Cd\frac{\eta}{\eps}\mathcal{L}^d(S^{\eps}).
 \end{equation}
As $S^{\eps+\eta} = (S^\eps)^\eta$ and $S^{\eps-\eta} \subset (S^{\eps})^{-\eta}$, this concludes the proposition.
\end{proof}

\bibliographystyle{amsplain}

{
\bibliography{References}
}

\end{document}